\documentclass{article}

\usepackage[preprint]{neurips_2026}

\usepackage[utf8]{inputenc} 
\usepackage[T1]{fontenc}    
\usepackage{url}            
\usepackage{booktabs}       
\usepackage{amsfonts}       
\usepackage{nicefrac}       
\usepackage{microtype}      
\usepackage{xcolor}         

\usepackage{color}
\usepackage{amsthm}
\usepackage{amssymb}
\usepackage{amsmath}
\usepackage{bbm}

\newtheorem{lemma}{Lemma}
\usepackage{amsthm}

\usepackage{algorithm}
\usepackage{algpseudocode}
\usepackage{caption} 
\usepackage{booktabs} 
\usepackage{multirow}

\usepackage{graphicx}
\usepackage{hyperref}

\usepackage{microtype}

\usepackage{inconsolata}
\usepackage{amsmath}
\usepackage{amssymb}

\usepackage{xcolor}  
\usepackage{url}
\usepackage{booktabs} 

\hypersetup{
	colorlinks=true, 
    citecolor=blue, 
    linkcolor=blue,
    urlcolor=blue,
	pdfborder={0 0 0},
}
\usepackage{xcolor}
\usepackage{tcolorbox}
\usepackage{listings}
\usepackage{tcolorbox}
\usepackage{xcolor}
\usepackage{float}
\usepackage{wrapfig} 

\definecolor{commentcolor}{RGB}{0,100,0}   
\definecolor{keywordcolor}{RGB}{0,0,255}   
\definecolor{stringcolor}{RGB}{163,21,21}  
\definecolor{identifiercolor}{RGB}{0,0,0}  

\newtcolorbox{mybox}[1][]{
    title=#1,
    fonttitle=\small,
    fontupper=\small,
    left=2mm,
    right=2mm,
    top=1mm,
    bottom=0mm,
}
\lstdefinestyle{mystyle}{
    commentstyle=\color{commentcolor},
    keywordstyle=\color{keywordcolor}\bfseries,
    stringstyle=\color{stringcolor},
    identifierstyle=\color{identifiercolor},
    basicstyle=\ttfamily\lst@ifdisplaystyle\tiny\fi,
    breakatwhitespace=false,
    breaklines=true,
    captionpos=b,
    keepspaces=true,
    numbers=none,
    numbersep=5pt,
    showspaces=false,
    showstringspaces=false,
    showtabs=false,
    tabsize=2,
    xleftmargin=0pt,
}
\newtheorem{theorem}{Theorem}

\newcommand{\argmax}{\mathop{\rm arg~max}\limits}

\title{Tail Competition Explains Scaling in Best-of-N Sampling for Verifiable Problems }

\author{%
  Taro Yano \\
  NEC Corporation\\
  \texttt{taro\_yano@nec.com} \\
  \And
  Yoichi Ishibashi \\
  NEC Corporation \\
  \texttt{yoichi-ishibashi@nec.com} \\
  \AND
  Masafumi Oyamada \\
  NEC Corporation \\
  \texttt{oyamada@nec.com} \\
}

\begin{document}

\maketitle

\begin{abstract}
Best-of-N sampling is a simple inference-time scaling strategy for large language models, yet its behavior under incomplete reward models remains poorly understood. We study Best-of-N as a competition between the extreme reward tails of correct and incorrect responses. This view shows that success depends not only on the probability of sampling a correct response, but also on whether correct responses dominate incorrect ones in the upper tail. Using extreme value theory, we characterize three asymptotic regimes: correct-tail-dominant, incorrect-tail-dominant, and critical explaining both monotonic gains and reward-hacking-induced degradation. We further derive finite-sample scaling laws for representative reward distributions, showing that, under Gaussian rewards, incorrect responses with sufficiently large variance can dominate selection even when their mean reward is lower. 
Unlike oracle selection, convergence with an incomplete selector can be polynomial rather than exponential. 
Finally, we propose a lightweight tail-tension criterion to estimate when additional sampling may become harmful. Experiments on synthetic data and real reward model scores are consistent with the theoretically predicted regimes, scaling laws, and diagnostic behavior, suggesting that tail competition provides a unified lens on Best-of-N sampling for verifiable problems.

\end{abstract}
\section{Introduction}

Large language models (LLMs)~\citep{DBLP:journals/corr/abs-2303-08774,anthropic2026claudeopus46,DBLP:journals/corr/abs-2505-09388,DBLP:journals/corr/abs-2407-21783,DBLP:journals/corr/abs-2512-02556,DBLP:journals/corr/abs-2412-08905,DBLP:journals/corr/abs-2512-20856,DBLP:journals/corr/abs-2503-19786,DBLP:journals/corr/abs-2310-06825,DBLP:journals/corr/abs-2507-20534,DBLP:journals/corr/abs-2506-13585} have achieved strong performance across reasoning, coding, and agentic tasks. A central approach to improving such systems at inference time is inference-time scaling including Chain-of-Thought~\citep{DBLP:conf/nips/Wei0SBIXCLZ22,DBLP:conf/nips/KojimaGRMI22,DBLP:journals/corr/abs-2408-03314,DBLP:journals/corr/abs-2507-06261}, self-refining methods~\citep{DBLP:conf/nips/MadaanTGHGW0DPY23,DBLP:conf/nips/ShinnCGNY23} or planning and acting~\citep{DBLP:conf/iclr/YaoZYDSN023}. 

One representative approach is Best-of-N (BoN) sampling. 
BoN is a simple yet powerful inference-time scaling method that generates multiple candidate responses from an LLM and selects the response with the highest score under a reward model.
\citet{DBLP:journals/corr/abs-2407-21787} and \citet{DBLP:conf/icml/SchaefferKHJPLJ25} represent one line of work that concisely illustrates why inference-time scaling is effective: even if a single inference does not yield a correct answer, performing $n$ inferences makes the probability that at least one correct answer is included scale exponentially in $n$ as $\exp(-\mathrm{e}^{-np})\approx 1 - \mathrm{e}^{-np}$. However, they do not discuss the rate at which performance scales when an incomplete reward function is used. Moreover, the performance of BoN does not necessarily improve monotonically with the number of candidates $n$. For example, reward hacking~\citep{DBLP:conf/icml/GaoSH23} is known to occur, where responses that receive high scores from the BoN reward model may nevertheless have low scores under the true reward model. Several existing studies have discussed reward hacking from the perspectives of reward-function error~\citep{DBLP:journals/corr/abs-2506-19248}, regret bounds~\citep{DBLP:journals/corr/abs-2507-05913}, and the relationship between true and proxy rewards~\citep{DBLP:journals/corr/abs-2603-06797}. However, when and why reward hacking arises remains insufficiently understood.

We analyze BoN for verifiable problems from the perspective of the reward-score distributions of correct and incorrect responses. Starting from the fact that BoN succeeds only when the maximum reward among correct responses exceeds that among incorrect responses~\citep{DBLP:journals/corr/abs-2511-18630}, we show that both the scaling rate and reward hacking can be explained by the competition between the upper tail of the score distribution for correct responses and the upper tail of the score distribution for incorrect responses. In particular, when the RM scores of correct and incorrect responses follow Gaussian distributions with means $\mu_c, \mu_w$ and variances $\sigma_c^2, \sigma_w^2$, respectively, we show that the scaling is polynomial of $n$ rather than exponential, and that reward hacking occurs when $\mu_c > \mu_w$ and $\sigma_c < \sigma_w$. This clearly indicates that the competition between correct and incorrect responses is governed by the means of the distributions when $n$ is small, whereas it is governed by their variances when $n$ is large. Thus, reward hacking can be understood as arising from the competition between two effects: performance improves for small $n$ due to $\mu_c > \mu_w$, while performance deteriorates for large $n$ due to $\sigma_w > \sigma_c$.

Furthermore, we connect the theory to a practical early stopping method. From the tail competition between the two distributions, the number of candidates $n^{\star}$ at which correct and incorrect outcomes are balanced can be derived analytically. However, at inference time, the correctness labels of individual candidate responses are unavailable, so the score distributions of the correct and incorrect classes cannot be used directly. We therefore fit a two-component Gaussian mixture distribution to the reward-model scores in an unsupervised manner, associating the high-mean component with correct responses and the low-mean component with incorrect responses.
Using the estimated distribution parameters, we compute $n^{\star}$ and propose an early stopping criterion that stops sampling before the planned budget $n$ when the hacking condition is satisfied. This method requires no correctness labels and can be applied by adding only lightweight distribution fitting to standard BoN inference.

The contributions of this work are as follows.
\begin{itemize}
    \item We characterize BoN as a tail competition between correct and incorrect reward distributions, yielding phase conditions including monotonic improvement and reward hacking.
    \item We derive finite-n scaling laws, showing that incomplete selection can converge polynomially rather than exponentially.
    \item We propose a label-free tail-tension early-stopping rule.
    \item We confirm the effectiveness of our theoretical results and the proposed early stopping method in experiments.
\end{itemize}
\section{Best-of-N Sampling in Verifiable Problems}
\subsection{Setting}
BoN sampling first generates $n$ responses $\{y_i\}_{i=1}^n$ to a problem $x$ using an LLM $\pi(\cdot\mid x)$. In this paper, we assume that $\{y_i\}_{i=1}^n$ are independent and identically distributed. Next, a reward model is used to compute a reward $r(x, y_i) \in \mathbbm{R}$ for each of the $n$ responses, and the response with the highest score is selected as the final answer $\hat{y}_{\mathrm{bon}}$ as follows:
\begin{eqnarray}
    \hat{y}_{\mathrm{bon}} = \argmax_{y \in \{y_i\}_{i=1}^n} r(x, y).
\end{eqnarray}

In this paper, we consider verifiable problems for which there exists a set of correct answers $Y^\star$ for a problem $x$. We define a response as correct when $y\in Y^\star$, and incorrect when $y\notin Y^\star$. For notational simplicity, we define the probability that BoN produces a correct answer for problem $x$, denoted by $S_n(x)$, as $S_n(x):=\mathbbm{E}_{\{y_i\sim\pi(\cdot\mid x)\}_{i=1}^n}\left[\Pr(\hat{y}_{\mathrm{bon}}\in Y^\star \mid x,\{y_i\}_{i=1}^n)\right]$.

Since $y$ generated by the LLM is a random variable following $\pi(\cdot \mid x)$, the reward value $r(x, y)$ itself is also a random variable. Here, we assume that the conditional cumulative distribution functions of the reward $s$ given whether the response is correct or incorrect,
$F_c(s\mid x):=\Pr(r(x,y)\le s \mid y\in Y^\star,x)$ and
$F_w(s\mid x):=\Pr(r(x,y)\le s \mid y\notin Y^\star,x)$,
exist and are continuous.

\subsection{Decomposition into Generation and Aggregation}
BoN can be decomposed into two phases: generating responses and selecting one from them. First, in the generation phase, let
$p(x):=\Pr_{y\sim\pi(\cdot\mid x)}(y\in Y^\star \mid x)$
be the probability that a single generation is correct, i.e., Pass@$1$. Then, when generating $n$ responses, we can expect approximately $n\cdot p(x)$ of them to be correct. Next, in the selection phase, when the generated responses contain $k$ correct answers, let
$A_n(k; x):=\Pr(\hat{y}_{\mathrm{bon}}\in Y^\star \mid x,k: \text{correct})$
denote the probability that the selection succeeds. Then the success probability can be decomposed according to the number of correct and incorrect responses generated by the generator as follows:
\begin{eqnarray}
    S_n(x)=\sum_{k=0}^n\binom{n}{k}p(x)^k(1-p(x))^{n-k}A_n(k;x).
\end{eqnarray}
This means that the number of correct answers produced by the generator fluctuates around $np(x)$, and $S_n$ is obtained by summing over these fluctuations. On the other hand, if the selector's accuracy does not change significantly under such fluctuations, then $A_n(\lfloor np(x)\rfloor; x)$ is almost equal to $S_n(x)$ as the following lemma.

\begin{lemma}[Reduction to the typical region]
\label{prop:typical_region_reduction}
Fix an input $x$ and assume that $0<p(x)<1$. Let $k_n:=\lfloor np(x)\rfloor$. Suppose that, for any $M>0$,
\[
\mathbb{E}_{k\sim \mathrm{Bin}(n, p(x))}\!\left[
|A_n(k;x)-A_n(k_n;x)|
\mathbf{1}\{|k-k_n|\le M\sqrt n\}
\right]
\;\xrightarrow[n\to\infty]{}\;0.
\]
Then the following holds:
\[
\limsup_{n\to\infty}|S_n(x)-A_n(k_n;x)|
\le
2e^{-2(M-1)^2}.
\]
\end{lemma}
The proof is given in Appendix~\ref{appendix:ass:typical_stability proof}. This result shows that, in the regime where $n$ is sufficiently large, the success probability of BoN, $S_n(x)$, is dominated by the selection accuracy in the typical region, where there are $np$ correct responses and $n(1-p)$ incorrect responses. Through this decomposition, if the success probability of BoN selection $A_n(k;x)$ does not vary rapidly with the number of correct responses near $k=k_n$, then BoN reduces to the typical-region selection performance. Therefore, in the following discussion, we focus on this quantity $A_n(k_n; x)$.

\section{Tail-Competition}\label{sec: tail-competition}
\begin{figure}[htb]
    \centering
    \includegraphics[width=0.6\linewidth]{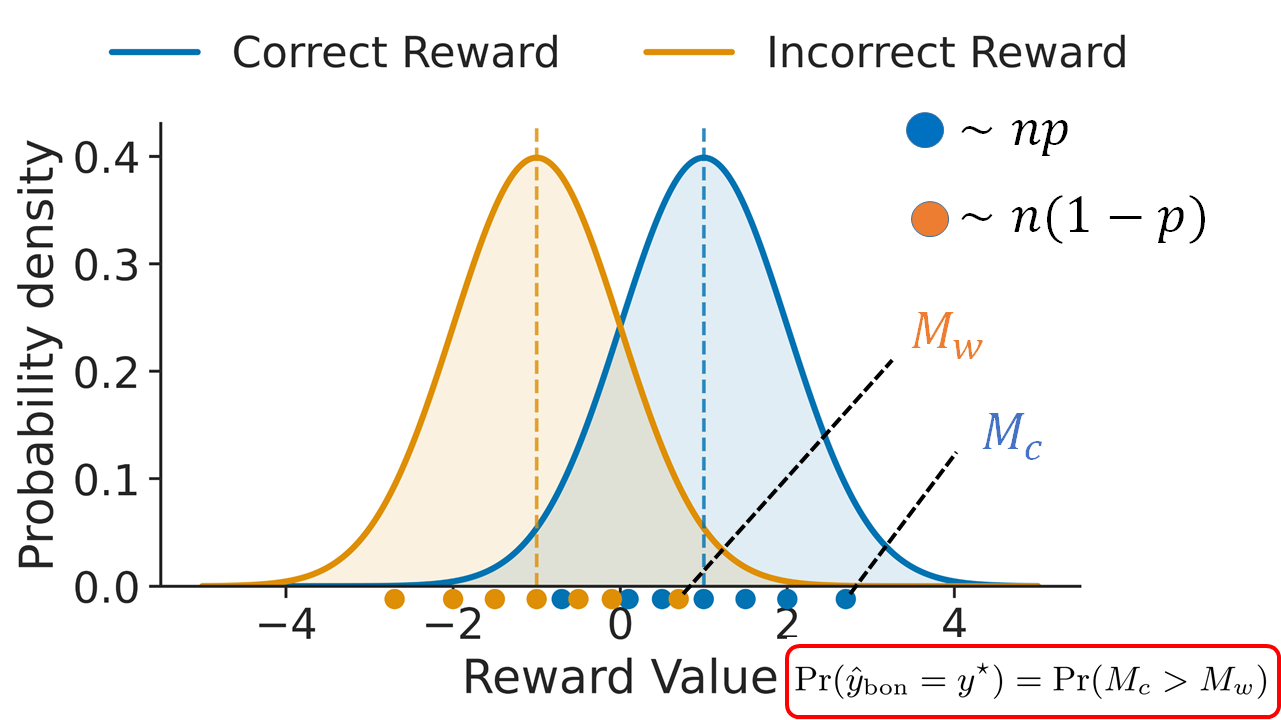}
    \caption{An illustrative overview of our explanation of BoN for verifiable problems. When \(n\) responses are sampled, assuming \(\mathrm{pass@1}=p\), we obtain approximately \(np\) correct responses and \(n(1-p)\) incorrect responses, each assigned a reward according to the corresponding reward distribution. Whether BoN selects a correct response is determined by the competition between the maximum reward \(M_c\) among the \(np\) correct responses and the maximum reward \(M_w\) among the \(n(1-p)\) incorrect responses. When \(n\) is sufficiently large, these maxima are governed by the tails of the respective distributions. Thus, the performance of BoN depends on the competition between the tails of the reward distributions for correct and incorrect responses.
}
    \label{fig:max_competition}
\end{figure}
In this section, we introduce a framework that views BoN as a competition between the reward maxima of correct and incorrect responses. By approximating the respective reward maxima using extreme value theory, we analyze how BoN accuracy changes as \(n\) grows. As a simple illustrative example, we consider the case where the rewards of correct and incorrect responses follow Gaussian distributions. However, the framework and results apply to almost any continuous distribution and are not limited to the Gaussian case.

\subsection{Extreme Value Theory Approximation}
It is known that the maximum value in a sample set of a random variable, after appropriate normalization, converges in distribution as $n \rightarrow \infty$ to one of three types of extreme value distributions. In this paper, we use this fact to analyze the asymptotic behavior when $n$ is large. To simplify notation, we write $p:=p(x), k_n:=\lfloor np\rfloor, m_n:=n-k_n$.
According to the Fisher--Tippett--Gnedenko theorem~\citep{Frechet1927EcartMaximum,FisherTippett1928LimitingForms,Gnedenko1943DistributionLimite}, if 
\begin{align}
\frac{M(m)-b(m)}{a(m)}
&\xrightarrow[m\to\infty]{} Z,
\end{align}
has non-degenerate limit for the maximum value $M(m)=\max_{x_j \in X} x_j$ of $m$ samples $\{x_i | x_i \sim P\}_{i=1}^m$ with appropriate $b(m)\in\mathbbm{R}, a(m)>0$, $Z$ follows one of three distributions: the Gumbel distribution, the Fréchet distribution, or the Weibull distribution depending on $P$. Light-tailed distributions whose tails decay exponentially, such as the gaussian distribution, converge to the Gumbel distribution; heavy-tailed distributions, such as the Pareto distribution, converge to the Fréchet distribution; and distributions with bounded random variables, such as the beta distribution, converge to the Weibull distribution. Therefore, for algorithms that depend only on the maximum value, such as BoN sampling, most distributions can be covered by examining the above three types of distributions, regardless of what distribution the reward distribution follows.

\paragraph{Example: Gaussian distribution.} When a random variable $X$ follows a gaussian distribution $X \sim\mathcal{N}(\mu, \sigma^2)$ with mean $\mu$ and variance $\sigma^2$,
\begin{eqnarray}
a(m) = \frac{\sigma}{\sqrt{2\log m} - \frac{\log\log m+\log(4\pi)}{2\sqrt{2\log m}}}, \hspace{3mm} b(m) = \mu + \sigma\cdot\left(\sqrt{2\log m} - \frac{\log\log m+\log(4\pi)}{2\sqrt{2\log m}}\right),
\label{eq:gaussian a_n, b_n}
\end{eqnarray}
then the scaled maximum value $\frac{M(m)-b(m)}{a(m)}$ of $m$ samples drawn from the gaussian distribution using the above $a(m), b(m)$ converges in distribution to the standard Gumbel distribution as $m\rightarrow \infty$.

Based on the Fisher--Tippett--Gnedenko theorem, this paper appropriately chooses $a_c(m;x)>0,b_c(m;x)$ and $a_w(m;x)>0,b_w(m;x)$ for correct and incorrect answers, respectively, and approximates
\begin{align}\label{eq:EVT approximation}
\frac{M_c(k_n)-b_c(k_n;x)}{a_c(k_n;x)}
&\approx Z_c,
&
\frac{M_w(m_n)-b_w(m_n;x)}{a_w(m_n;x)}
&\approx Z_w,
\end{align}
as above.

\subsection{Classification of BoN Behaviors}
In this section, we show that the correctness probability of BoN sampling converges, for general reward distributions, to three regimes as $n\rightarrow\infty$, together with the corresponding conditions, and how these regimes are realized when specific reward distributions are assumed.

First, to simplify notation, define
\[
\alpha_n(x):=\frac{a_c(k_n;x)}{a_w(m_n;x)},
\qquad
\beta_n(x):=\frac{b_c(k_n;x)-b_w(m_n;x)}{a_w(m_n;x)}.
\]
\begin{theorem}[Behavior for infinite $n$]
\label{theorem:generalized_three_regimes}
In the typical region $k_n=\lfloor np\rfloor$, the extreme value approximation gives
\[
A_n(k_n;x) = \Pr(M_c(k_n)>M_w(m_n))
\approx
\Pr\!\bigl(\alpha_n(x)Z_c-Z_w>-\beta_n(x)\bigr).
\]
Then the asymptotic behavior of $A_n(k_n;x)$ is classified into the following three regimes.

\textbf{(i) Correct-side extreme-value-dominant phase.}
If any of the following holds:
(a) $\alpha_n\to\alpha\in(0,\infty)$ and $\beta_n\to+\infty$,
(b) $\alpha_n\to\infty$ and $\beta_n/\alpha_n\to+\infty$,
or
(c) $\alpha_n\to0$ and $\beta_n\to+\infty$,
then $A_n(k_n;x)\to1$.

\textbf{(ii) Incorrect-side extreme-value-dominant phase.}
If any of the following holds:
(a) $\alpha_n\to\alpha\in(0,\infty)$ and $\beta_n\to-\infty$,
(b) $\alpha_n\to\infty$ and $\beta_n/\alpha_n\to-\infty$,
or
(c) $\alpha_n\to0$ and $\beta_n\to-\infty$,
then $A_n(k_n;x)\to0$.

\textbf{(iii) Critical phase.}
If
(a) $\alpha_n\to\alpha\in(0,\infty)$ and $\beta_n\to\beta\in\mathbb{R}$, then
$A_n(k_n;x)\to \Pr(\alpha Z_c-Z_w>-\beta)$;
(b) $\alpha_n\to\infty$ and $\beta_n/\alpha_n\to\tau\in\mathbb{R}$, then
$A_n(k_n;x)\to \Pr(Z_c > -\tau)$;
or
(c) $\alpha_n\to0$ and $\beta_n\to\beta\in\mathbb{R}$,
then $A_n(k_n;x)\to Pr(Z_w < \beta)$.
\end{theorem}
See Appendix \ref{proof:generalized_three_regimes} for the proof. 

\paragraph{Example: Gaussian distribution.} Assume that the rewards for correct and incorrect answers follow gaussian distributions with means $\mu_c, \mu_w$ and variances $\sigma_c^2, \sigma_w^2$, respectively. In this case, by taking $a_n, b_n$ as in Eq.~\ref{eq:gaussian a_n, b_n}, the asymptotic behavior as $n \rightarrow \infty$ is, by Theorem~\ref{theorem:generalized_three_regimes},
\begin{itemize}
    \item when $\sigma_c > \sigma_w$, or when $\sigma_c = \sigma_w, \mu_c > \mu_w$, the correct-side extreme-value-dominant phase,
    \item when $\sigma_c < \sigma_w$, or when $\sigma_c = \sigma_w, \mu_c < \mu_w$, the incorrect-side extreme-value-dominant phase,
    \item when $\sigma_c = \sigma_w, \mu_c = \mu_w$, the critical phase.
\end{itemize}
\begin{proof}
First, $\alpha_n \rightarrow \frac{\sigma_c}{\sigma_w} < \infty$, so it converges to a finite value. Since the largest dominant term of $\beta_n$ is $\frac{\sigma_c-\sigma_w}{\sigma_w}\cdot 2\log{n}$, the phase becomes the correct-side extreme-value-dominant phase when the sign of $\sigma_c-\sigma_w$ is positive, and the incorrect-side extreme-value-dominant phase when it is negative. When it is 0, $\frac{\mu_c-\mu_w}{\sigma_w}\sqrt{2\log n}$ becomes the dominant term. Therefore, if $\mu_c > \mu_w$, the phase becomes the correct-side extreme-value-dominant phase; if $\mu_c < \mu_w$, the phase becomes the incorrect-side extreme-value-dominant phase; and only when $\mu_c = \mu_w$, that is, when the reward distributions for correct and incorrect answers are exactly the same, does it converge to the critical phase.    
\end{proof}
In the case of the gaussian distribution, as $n \rightarrow \infty$, the maximum reward increases according to $b(m)$, namely $\sigma\sqrt{2\log{m}}$, and values in the right tail of the reward distributions for correct and incorrect answers are obtained; the side with the thicker right tail dominates. This means that even if the mean reward for correct answers is higher than that for incorrect answers, if the variance of incorrect answers is higher, the accuracy converges to 0 as $n\rightarrow\infty$. Thus, it is important for the success of BoN that the right tail of the distribution is properly optimized. Moreover, since the critical phase holds only when the correct and incorrect answers have exactly the same distribution, it is considered to occur very rarely in practice.

\paragraph{Four phases of $n$-dependence.}
The behavior of BoN as a function of $n$ is determined by two factors:
the initial direction from $n=1$ to $n=2$, and the asymptotic tail-dominant
side as $n\to\infty$. The former is governed by
$\Pr(R_c>R_w\mid x)$, while the latter is determined by Theorem~1.
This yields the four phases summarized in Table~\ref{tab:four-phases}.

\begin{table}[t]
\centering
\caption{Four phases of BoN $n$-dependence. The initial trend is determined by
$\Pr(R_c>R_w\mid x)$, while the asymptotic behavior is determined by the
tail-dominant side in Theorem~1.}
\label{tab:four-phases}
\begin{tabular}{c c c c}
\toprule
Phase & Initial trend & Asymptotic limit & Typical behavior \\
\midrule
iA  & $S_2>S_1$ & $S_n\to 1$ & monotonic improvement \\
iB  & $S_2<S_1$ & $S_n\to 1$ & delayed improvement (e.g. U-shape) \\
iiA & $S_2>S_1$ & $S_n\to 0$ & ``reward hacking'' (e.g. inverted U-shape) \\
iiB & $S_2<S_1$ & $S_n\to 0$ & monotonic degradation \\
\bottomrule
\end{tabular}
\end{table}
Here, $S_2>S_1$ is equivalent to $\Pr(R_c>R_w\mid x)>1/2$.
Thus, small-$n$ improvement is controlled by pairwise correct--incorrect
comparison, whereas large-$n$ behavior is controlled by the extreme upper
tails. The most important non-monotonic case is phase iiA: BoN initially
improves, but eventually degrades because the incorrect-answer tail dominates.
We refer to this regime as reward hacking.


\paragraph{Example: Gaussian distribution.}

\begin{figure}[t]
    \centering

    \noindent
    \makebox[\linewidth]{%
        \hrulefill
        \hspace{0.8em}
        \textbf{Reward Model with Monotonic Improvement}
        \hspace{0.8em}
        \hrulefill
    }

    \vspace{0.8em}

    \begin{minipage}{0.48\linewidth}
        \centering
        \textbf{Small $n$}

        \vspace{0.4em}

        \includegraphics[width=\linewidth]{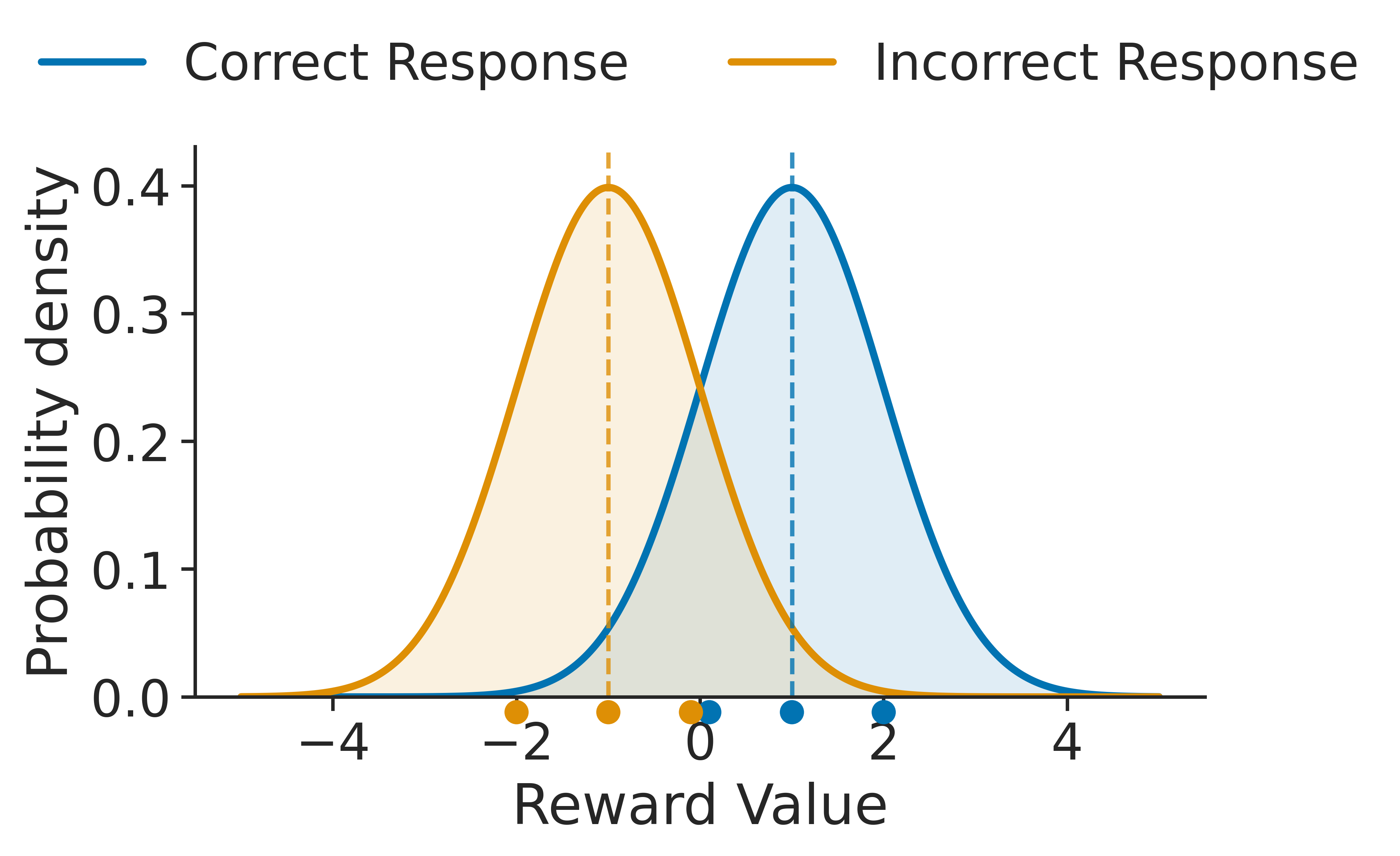}
    \end{minipage}
    \hfill
    \begin{minipage}{0.48\linewidth}
        \centering
        \textbf{Large $n$}

        \vspace{0.4em}

        \includegraphics[width=\linewidth]{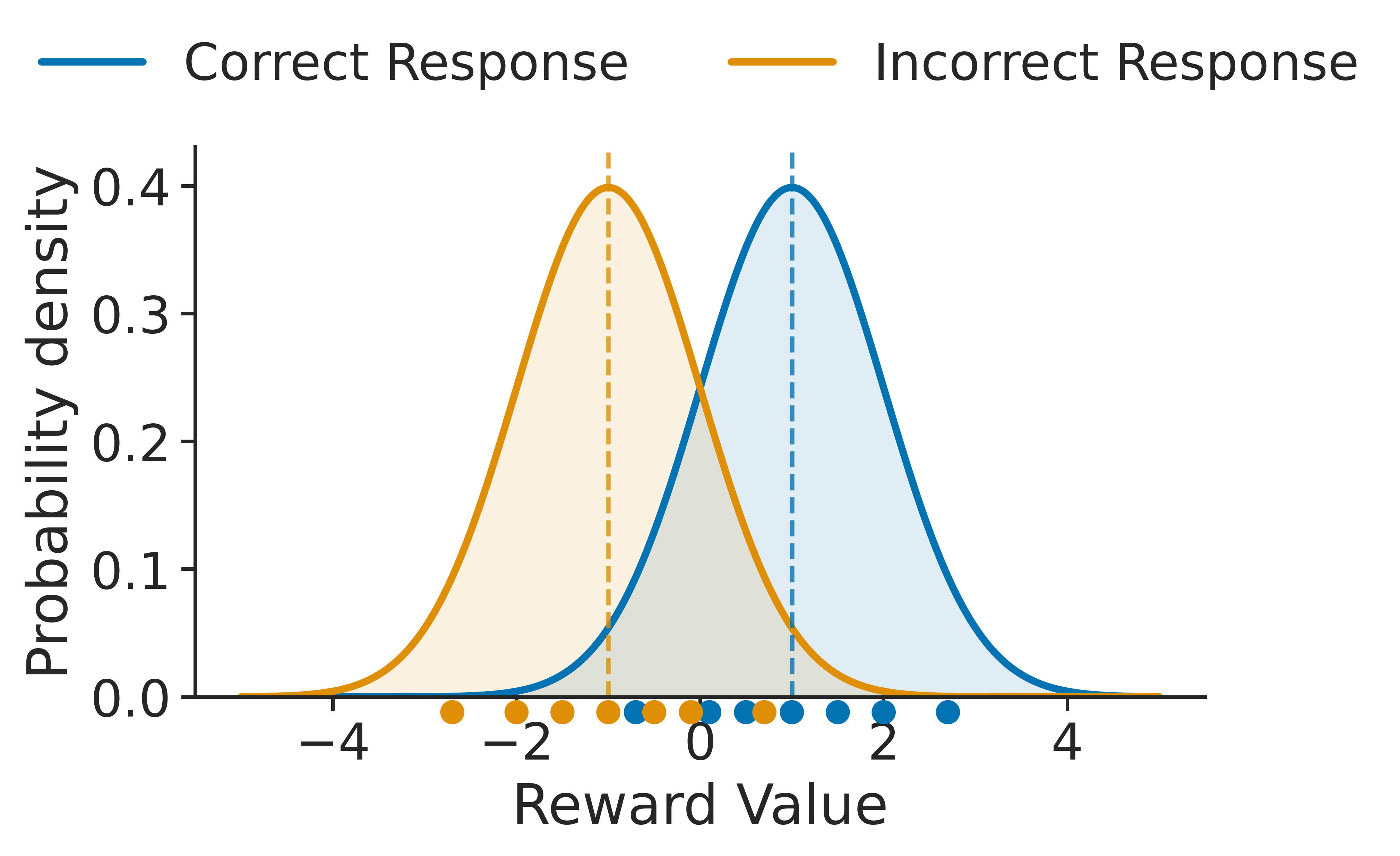}
    \end{minipage}

    \vspace{1.2em}

    \noindent
    \makebox[\linewidth]{%
        \hrulefill
        \hspace{0.8em}
        \textbf{Reward Model with Reward Hacking}
        \hspace{0.8em}
        \hrulefill
    }

    \vspace{0.8em}

    \begin{minipage}{0.48\linewidth}
        \centering
        \textbf{Small $n$}

        \vspace{0.4em}

        \includegraphics[width=\linewidth]{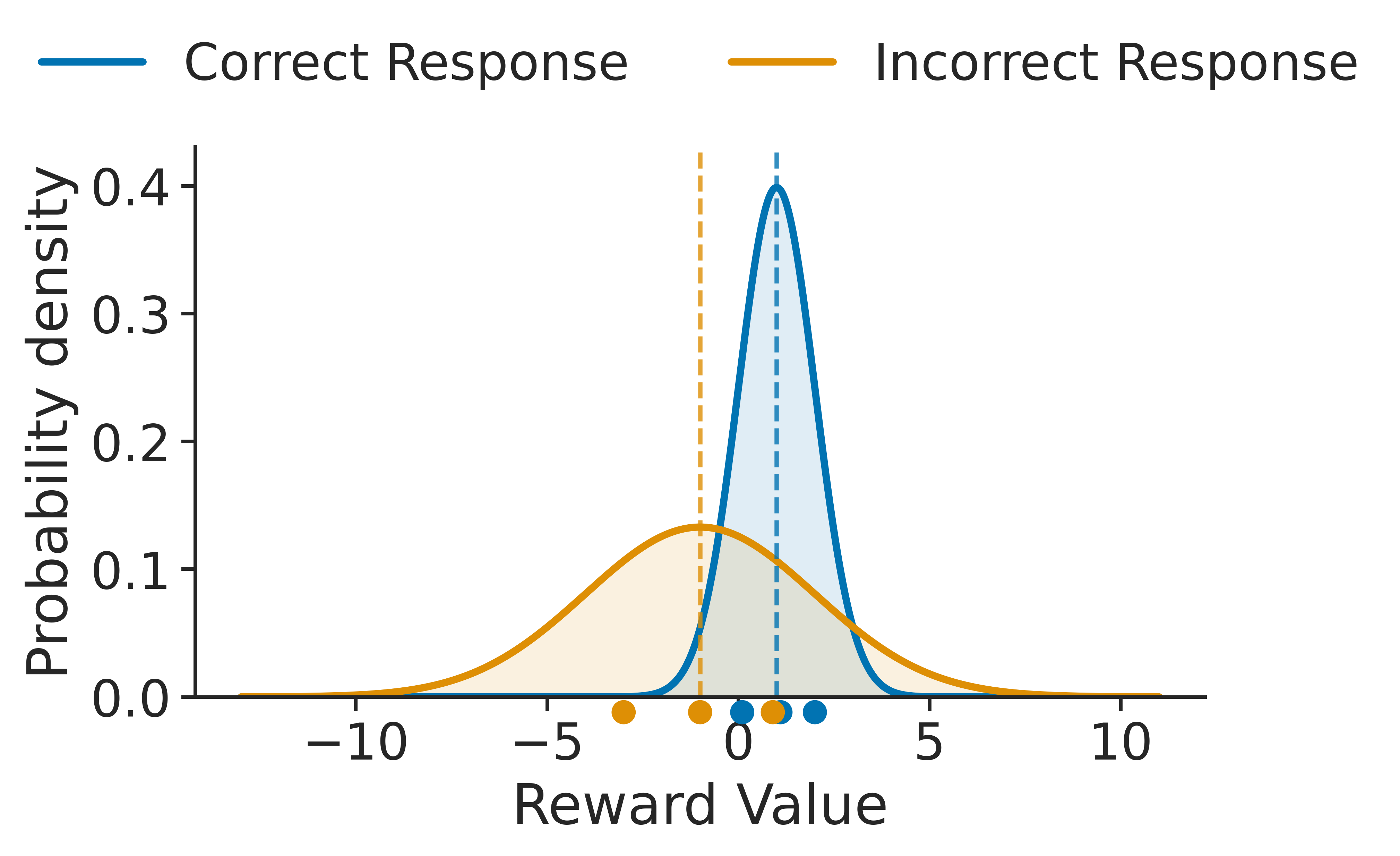}
    \end{minipage}
    \hfill
    \begin{minipage}{0.48\linewidth}
        \centering
        \textbf{Large $n$}

        \vspace{0.4em}

        \includegraphics[width=\linewidth]{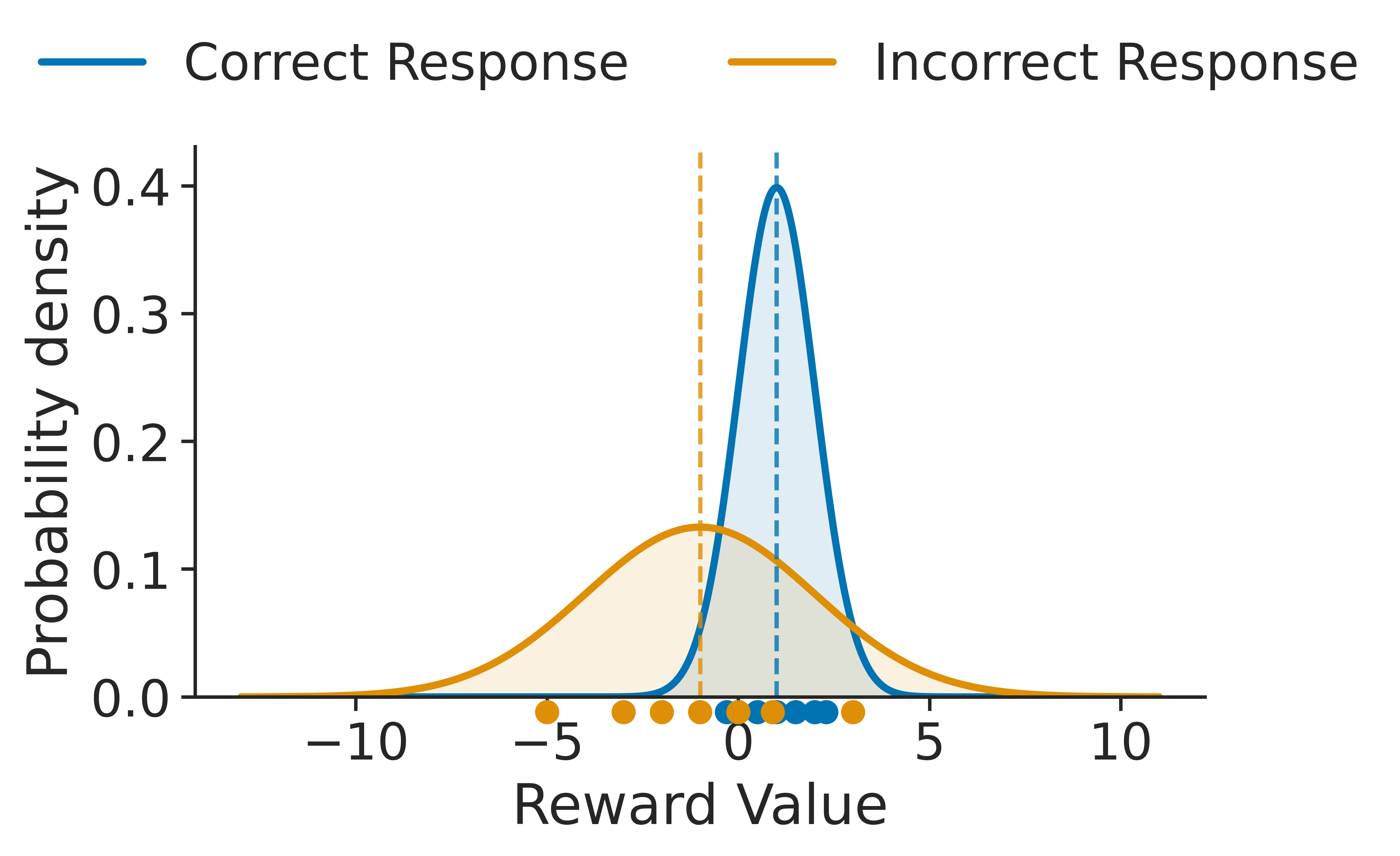}
    \end{minipage}

    \caption{An illustrative explanation of reward hacking in verifiable problems. 
    The blue and orange points represent correct and incorrect sampled responses, respectively, with $n$ points in total. The response outlined in red is the one selected by BoN. Thus, if a blue response is outlined in red, BoN successfully selects a correct response. In the monotonic-improvement case (top row), correct and incorrect responses have reward distributions with the same variance, so their maximum rewards grow at similar rates as $n$ increases. In the reward-hacking case (bottom row), incorrect responses have larger reward variance, causing their maximum reward to grow faster and eventually exceed that of correct responses. Thus, BoN selects a correct response at small $n$ but incorrect one at large $n$.}
    \label{fig:monotonic_reward_hacking}
\end{figure}

For $n=2$, the condition under which BoN accuracy improves is
$
\Pr(R_c>R_w\mid x)
=
\Phi\!\left(
\frac{\mu_c-\mu_w}{\sqrt{\sigma_c^2+\sigma_w^2}}
\right)
$,
which implies $\mu_c>\mu_w$, where $\Phi$ is the cumulative distribution function of gaussian distribution.
Moreover, the incorrect side is upper-tail dominant if and only if $\sigma_w>\sigma_c$.
Thus, the Gaussian reward-hacking regime is
\begin{align}
\mu_c>\mu_w
\qquad\text{and}\qquad
\sigma_w>\sigma_c .
\label{eq:gaussian_reward_hacking_condition}
\end{align}
An illustration is shown in figure~\ref{fig:monotonic_reward_hacking}.

In this regime, the characteristic turnover scale is obtained by equating the
first-order extreme-value approximations of the correct and incorrect maxima: $\mu_c+\sigma_c\sqrt{2\log n}\approx\mu_w+\sigma_w\sqrt{2\log n}$.
Thus, we define the turn-over $n$ as follows.
\begin{align}
n^{\star}(x)
:=
\exp\!\left(
\frac{(\mu_c-\mu_w)^2}
{2(\sigma_w-\sigma_c)^2}
\right),
\label{eq:gaussian_ev_scale}
\end{align}
where, $n^{\star}(x)$ is a leading-order extreme-value crossover scale: it indicates where the incorrect maximum reward catches up with the correct maximum reward, rather than the exact optimizer of $S_n(x)$.

\section{Scaling Law}
In this section, in the correct-side extreme-value-dominant or critical phase, we analyze how the speed of asymptotic convergence to accuracy 1, that is, the $n$-scaling, differs between the oracle selection function and an incomplete selection function. In general, the scaling behavior changes depending on the reward distribution. This section describes the universal approximate behavior for three domains of attraction: the Gumbel distribution, the Fréchet distribution, and the Weibull distribution.

To derive the scaling law for finite $n$, we first evaluate the error due to the EVT approximation. Given the CDF errors $\epsilon_c(n), \epsilon_w(n)$ induced by the EVT approximations of the reward distributions for correct and incorrect answers, the following lemma holds for the EVT approximation of $A_n$, defined as $\tilde{A}_n(k_n):= \Pr(a_c(k_n)Z_c(k_n) + b_c > a_w(m_n)Z_w(m_n) + b_w)$.
\begin{lemma}[EVT approximation error bound]
\label{prop:evt_approximation_error_bound}
When the CDF errors are given by
\begin{eqnarray}
\epsilon_c(n)&:=&\sup_z |\Pr(M_c(k_n)\leq z) - \Pr(a_c(k_n)Z_c(k_n) + b_c\leq z)|,\\
\epsilon_w(n)&:=&\sup_z |\Pr(M_w(m_n)\leq z) - \Pr(a_w(m_n)Z_w(m_n) + b_w\leq z)|,
\end{eqnarray}
we have
\begin{eqnarray}
|A_n(k_n) - \tilde{A}_n(k_n)| \leq \epsilon_c(n) + \epsilon_w(n).
\end{eqnarray}
\end{lemma}
The proof is provided in Appendix~\ref{appendix:evt_approximation_error_bound}. We will derive the scaling law of $\tilde{A}_n$ in the rest of this section.

\begin{theorem}[Finite-$n$ scaling in the Gumbel domain of attraction]\label{theorem:gumbel scaling}
Assume that the reward distribution belongs to the Gumbel domain of attraction and that $\alpha_n \rightarrow \alpha < \infty$ and $\beta_n \rightarrow \infty$ hold. Then, it falls into the correct-side extreme-value-dominant phase and 
\begin{eqnarray}
    \tilde{A}_n(k_n; x) = 1 - \Gamma(1 + \alpha)\mathrm{e}^{-\beta_n} + o(\mathrm{e}^{-\beta_n}),
\end{eqnarray}
holds.
\end{theorem}

\paragraph{Example: Gaussian distribution.} If the reward follows a Gaussian distribution and belongs to the correct-side extreme-value-dominant phase, then $\alpha_n \rightarrow \frac{\sigma_c}{\sigma_w} < \infty$ and $\beta_n \rightarrow \infty$, and
\begin{eqnarray}
\tilde{A}_n(k_n;x) 
\approx
1 -\Gamma\left(1+\frac{\sigma_c}{\sigma_w}\right)
p^{-\frac{\sigma_c}{\sigma_w}}
(1-p)^{2-\frac{\sigma_c}{\sigma_w}}
n^{-\gamma_n},\label{eq:gaussian scaling law}    
\end{eqnarray}
where
$
\gamma_n
=
2\left(\frac{\sigma_c}{\sigma_w}-1\right)
+
\frac{\mu_c-\mu_w}{\sigma_w}
\sqrt{\frac{2}{\log n}}
-
\left(\frac{\sigma_c}{\sigma_w}-1\right)
\frac{\log(4\pi\log n)}{\log n}$. This shows that, up to logarithmic and sub-polynomial corrections, the convergence is governed by the polynomial rate, which is slower than the exponential asymptotic convergence $1-\mathrm{e}^{-np}$ obtained when using the oracle selection function. On the other hand, it is well known that the convergence of the scaled Gaussian maximum to the Gumbel distribution is slow. Consequently, when the EVT approximation error is taken into account, while the precise scaling law is still polynomial in \(n\), but with a different exponent. Further details are provided in Appendix~\ref{appendix:precise gaussian scaling}.

\begin{theorem}[Finite-$n$ scaling in the Fr\'echet domain of attraction]
\label{theorem:frechet scaling}
Assume that the reward distribution belongs to the Fr\'echet domain of attraction with shape parameter $\rho_c, \rho_w$. Furthermore, we assume that $b_m(x)=0, \alpha_n(x):=\frac{a_c(k_n;x)}{a_w(m_n;x)}\to\infty$ holds. Then,
\begin{align}
\tilde{A}_n(k_n;x)
&=
1-\Gamma\!\Bigl(1+\frac{\rho_w}{\rho_c}\Bigr)\alpha_n(x)^{-\rho_w}
+o(\alpha_n(x)^{-\rho_w}),
\end{align}
holds.
\end{theorem}

\begin{theorem}[Finite-$n$ scaling in the Weibull domain of attraction]
\label{theorem:weibull scaling}
Assume that the reward distribution belongs to the Weibull domain of attraction with shape parameter $\xi_c<0, \xi_w<0$ and share a common finite upper endpoint: $r_c^{\max}=r_w^{\max}=:r^{\max}$. Choose the location normalizations as $b_c(k_n)=b_w(m_n)=r_{\max}$ and assume $\alpha_n(x):=\frac{a_c(k_n;x)}{a_n(m_n;x)}\to0$. Then,
\begin{align}
\tilde{A}_n(k_n;x)
&=
1-\Gamma\!\Bigl(1+\frac{\xi_c}{\xi_w}\Bigr)\alpha_n(x)^{-\frac{1}{\xi}}
+o(\alpha_n(x)^{-\frac{1}{\xi}}).
\end{align}
\end{theorem}

See section~\ref{appendix:proof:gumbel}, \ref{appendix:proof:frechet}, and \ref{appendix:proof:weibull} for the proofs.

\section{Tail-Tension Early Stopping}
\label{sec:tail_tension_stopping}

Motivated by the tail-competition analysis in Section~\ref{sec: tail-competition},
we propose \emph{tail-tension early stopping}, an online rule that stops BoN
sampling before the planned budget $n$ when additional samples are predicted
to be harmful. After observing the reward scores of the first $t$ candidates, we fit a two-component Gaussian mixture to the per-query score distribution,
\begin{equation}
r \sim \pi_h \mathcal{N}(\mu_h,\sigma_h^2)
      + \pi_l \mathcal{N}(\mu_l,\sigma_l^2),
\end{equation}
where the high-mean component is used as a proxy for correct responses and the
low-mean component as a proxy for incorrect responses:
$(\mu_c,\sigma_c, p)\gets(\mu_h,\sigma_h,\pi_h)$ and
$(\mu_w,\sigma_w, 1-p)\gets(\mu_l,\sigma_l,\pi_l)$. This assignment is only a proxy, and we discuss the potential issues caused by misassignment in Appendix~\ref{appendix:misassignment}. Furthermore, we show that misassignment is rare in practice because the reward model is sufficiently optimized and introduce a conservative guard against potential misassignment.

Early stopping is considered only when the low-mean component has larger
variance, $\sigma_l>\sigma_h$. Under the Gaussian analysis, this corresponds
to $\sigma_w>\sigma_c$, where the incorrect-answer tail can dominate for large
$n$ and cause reward hacking. 
Under this condition, the stopping point is estimated as
\begin{equation}
\hat{n}^\star
\approx
\exp\left(
\frac{(\mu_h-\mu_l)^2}
     {2(\sigma_l-\sigma_h)^2}
\right).
\end{equation}
To account for the dependence on $n$ beyond leading order, we can numerically solve $\mu_h+\sigma_h\sqrt{2\log (n\pi_h)} = \mu_l+\sigma_l\sqrt{2\log (n\pi_l)}$. Since we observed very little change in accuracy, we do not include corrections beyond the leading-order approximation.

Since the extreme-value approximation is unreliable for very small budgets, we
do not stop before a minimum budget $n_{\mathrm{th}}$. The online stopping rule is therefore
\begin{equation}
\tau =
\min\left\{
t\le N:
t\ge n_{\mathrm{th}},\;
\sigma_l>\sigma_h,\;
\hat{n}^\star \le t
\right\},
\end{equation}
with $\tau=n$ if the set is empty. Equivalently, sampling continues until either
the planned budget is reached or the estimated optimal budget has been reached
under the tail-tension condition. The final response is selected from the
generated candidates as $\hat{y}=\arg\max_{1\le i\le \tau} r(x,y_i)$.
The overhead is only a lightweight mixture fit and the computation of
$\hat{n}^\star$. Since the rule is online, it saves both generation and reward-model
inference cost whenever $\tau<n$.
\section{Experiments}\label{sec:experimental results}
In this section, we experimentally examine the following points.
(1) What kind of distribution do real-world reward models follow?
(2) Whether the conditions for the four-phase classification and the predicted convergence destinations are correct
(3) The validity of the scaling law
(4) The effectiveness of Tail-Tension Early Stopping

\paragraph{Setting.}
We use the 10,000-response-per-problem dataset~\footnote{\url{https://huggingface.co/datasets/ScalingIntelligence/monkey_business}}~\citep{DBLP:journals/corr/abs-2407-21787}, which covers 2,878 problem--generator pairs across MATH~\citep{DBLP:conf/nips/HendrycksBKABTS21}, MiniF2F~\citep{DBLP:conf/iclr/ZhengHP22}, GSM8K~\citep{DBLP:journals/corr/abs-2110-14168}, and CodeContests~\citep{DBLP:journals/corr/abs-2203-07814}. In this study, we assigned rewards to the 10,000 generated responses for each of the 2,878 (problem, model) pairs and used them for response selection in BoN sampling. To measure the accuracy of BoN for each $n$, we sampled $n$ responses from the 10,000 responses, simulated this process 1,000 times, and report the average. The reward models used in our experiments are Skywork series~\footnote{\url{https://huggingface.co/collections/Skywork/skywork-reward-v2}}~\citep{DBLP:journals/corr/abs-2507-01352} including Skywork-Reward-V2-Llama-3.1-8B, Skywork-Reward-V2-Qwen3-8B, Skywork-Reward-V2-Qwen3-4B, and Skywork-Reward-V2-Qwen3-1.7B, Llama-3.1-8B-Instruct-RM-RB2~\citep{DBLP:journals/corr/abs-2506-01937}\footnote{\url{https://huggingface.co/allenai/Llama-3.1-8B-Instruct-RM-RB2}}, and RAMO-Llama3.1-8B\footnote{\url{https://huggingface.co/HFXM/RAMO-Llama3.1-8B}}, were selected from top-ranked models in RewardBench 2~\citep{DBLP:journals/corr/abs-2506-01937}. We used vLLM~\citep{kwon2023efficient} for reward model inference, and all experiments were conducted on NVIDIA L40S (48GB) GPUs.

\paragraph{Overall results.} Figure~\ref{fig:reward-vs-oracle} shows representative real-data results. Oracle Pass@$n$ improves rapidly with $n$, whereas BoN with incomplete reward models often improves more slowly, saturates, or degrades. This gap confirms that generation-side scaling, i.e., Pass@$n$ theory alone does not explain BoN performance; selection depends critically on the reward-score tails.

\begin{figure}[htb]
  \centering 

  \begin{minipage}{0.3\linewidth}
    \centering
    \includegraphics[width=\linewidth]{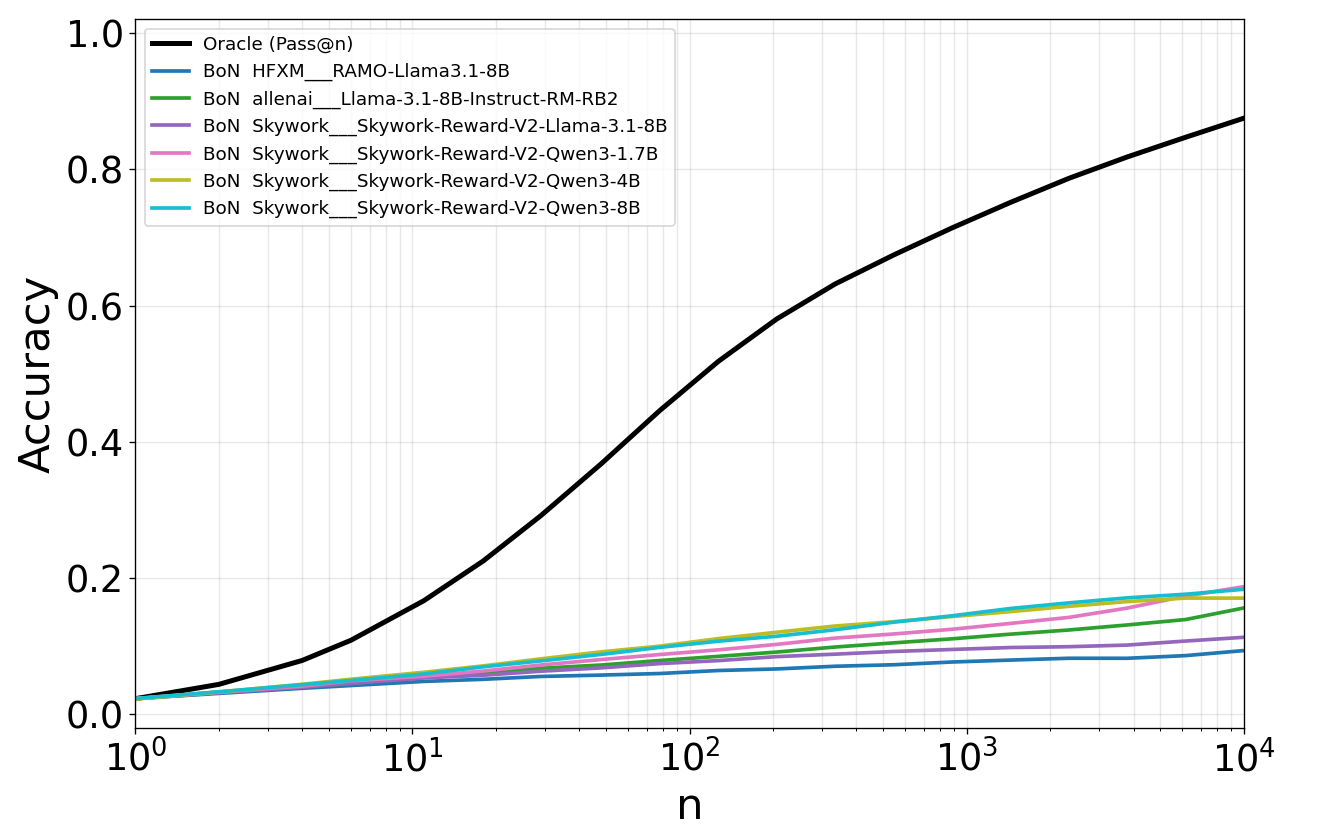}
  \end{minipage}
  \begin{minipage}{0.3\linewidth}
    \centering
    \includegraphics[width=\linewidth]{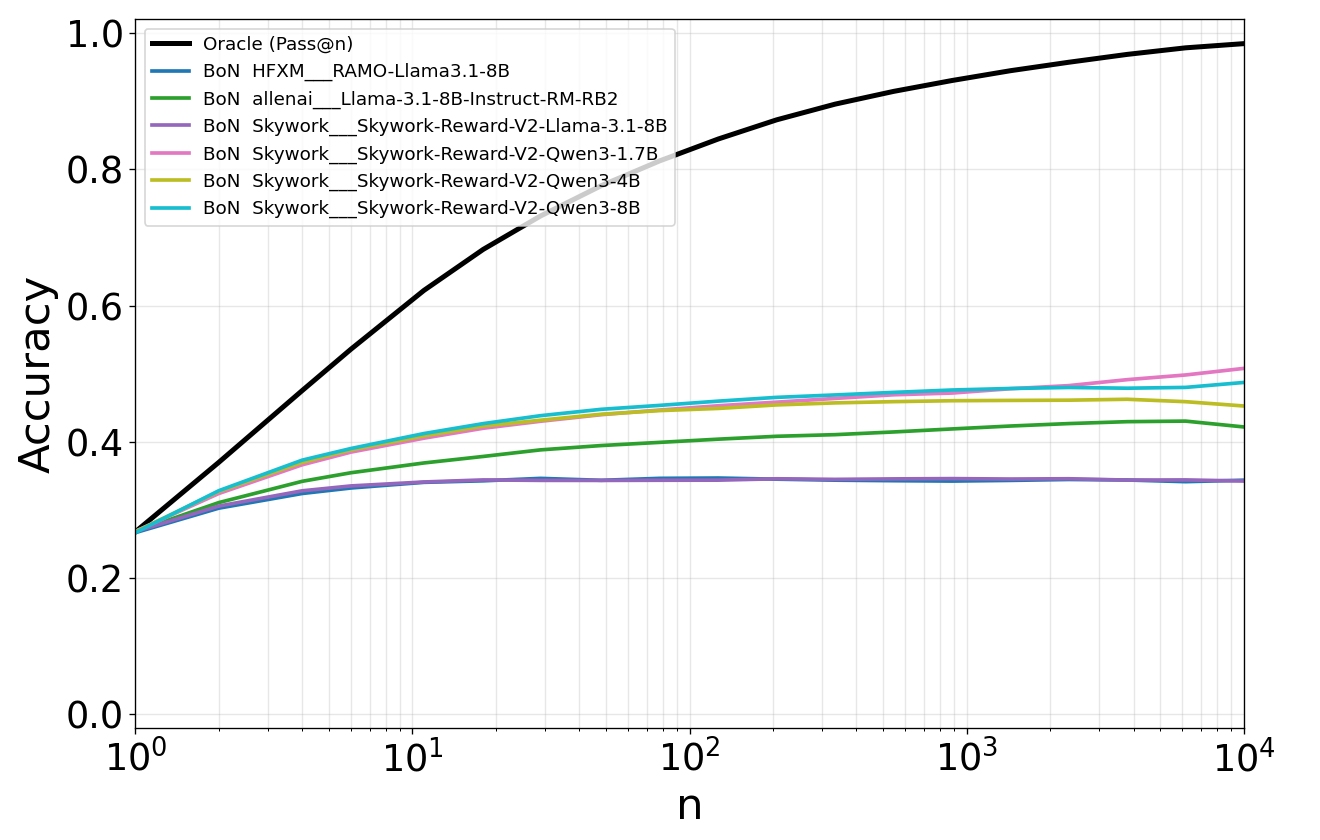}    
  \end{minipage}
  \begin{minipage}{0.3\linewidth}
    \centering
    \includegraphics[width=\linewidth]{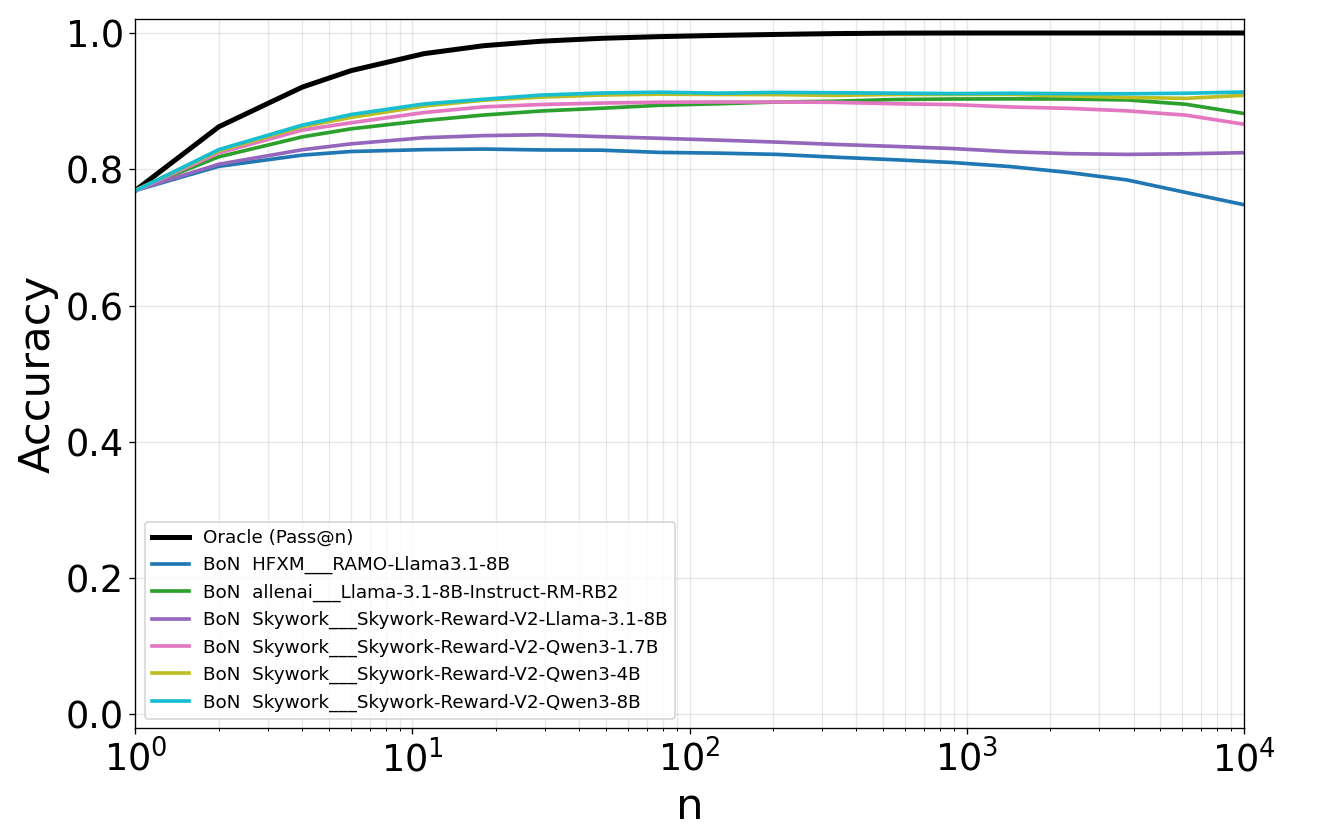}    
  \end{minipage}

  \caption{
    Overall results for MATH with Pythia-12B, MATH with Llama-3-8B-Instruct, and GSM8K with Llama-3-8B-Instruct. The gap between BoN and oracle selection often widens with $n$, and some settings exhibit non-monotonic BoN accuracy.
  }
  \label{fig:reward-vs-oracle}
\end{figure}

\paragraph{What kind of distribution do real-world reward models follow?} Our theory provides a framework applicable to general continuous distributions with max-stability, but in order to confirm that actual reward distributions fall into such domains, we use MATH-Pythia 12B combination and fit the distributions with max-stability using maximum likelihood and evaluate the goodness of fit using the Anderson--Darling test (AD test), log-likelihood, AIC~\citep{akaike1974new}, and BIC~\citep{schwarz1978estimating}. The results are shown in Table~\ref{tab:experiment:1}. From the pass rates of the AD test, we can see that real-world reward distributions can reasonably modeled by distributions with max-stability. 
Table~\ref{tab:experiment:2} summarizes reward-score statistics for correct and incorrect responses. Across settings, correct responses have higher mean scores, while incorrect responses often have larger variance, especially except for Sky-Qwen-4B. This mean--variance pattern suggests tail tension:
correct responses are favored on average, but incorrect responses may dominate the high-reward tail and can be evidence that reward hacking frequently occurs in real-world scenarios.

\paragraph{Verification of theoretical results using synthetic data with Gaussian reward distributions.} To verify the theory in a clean setting, we use synthetic data in which correct and incorrect responses follow Gaussian distributions with different mean and variance parameters, and pass@1 is assumed to be $p=[0.05, 0.1, 0.3, 0.5]$. 
Figure~\ref{fig:four phase} shows whether the four-phase classification and the theoretical $n^{\star}$ agree.
The four panels from left to right correspond to $(\mu_c-\mu_w, \sigma_c^2 - \sigma_w^2) \rightarrow (+, +), (-, -), (+, -), (-, +)$, respectively, satisfying the theoretical requirements of phase iA, phase iiB, phase iiA (reward hacking), and phase iB, respectively, and we confirm that they indeed match the expected behavior. We also see that the theoretically predicted $n^{\star}$ agrees well.
\begin{figure}[htb]
  \centering 

  \begin{minipage}{0.24\linewidth}
    \centering
    \includegraphics[width=\linewidth]{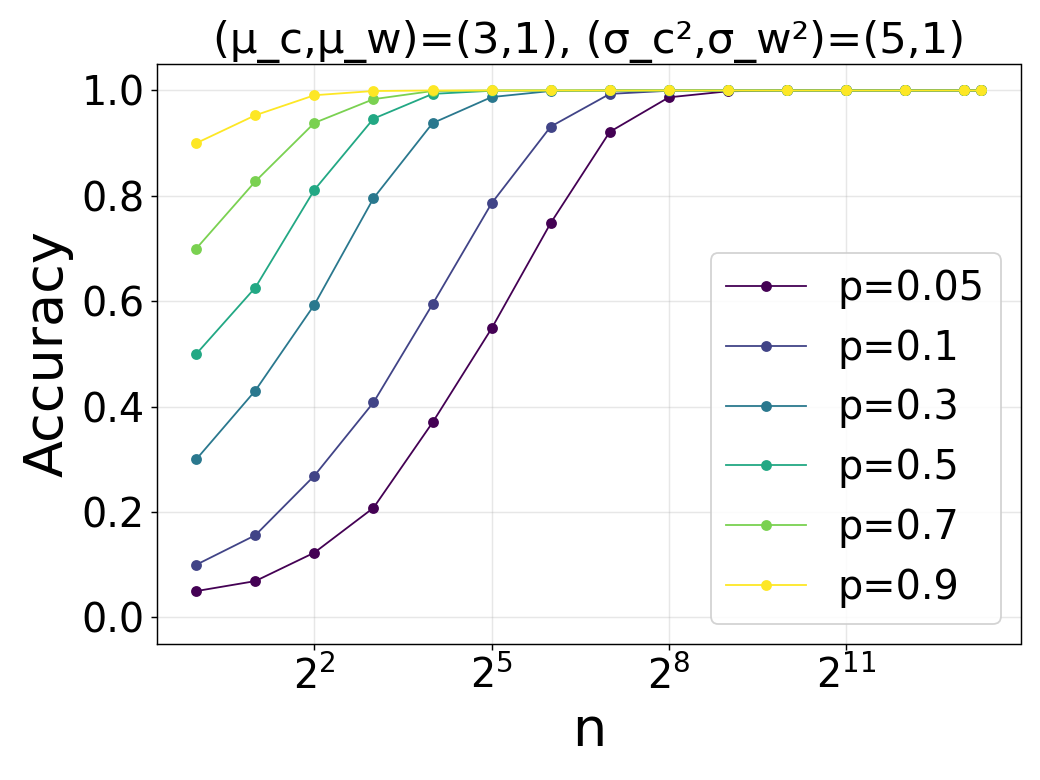}
  \end{minipage}  
  \begin{minipage}{0.24\linewidth}
    \centering
    \includegraphics[width=\linewidth]{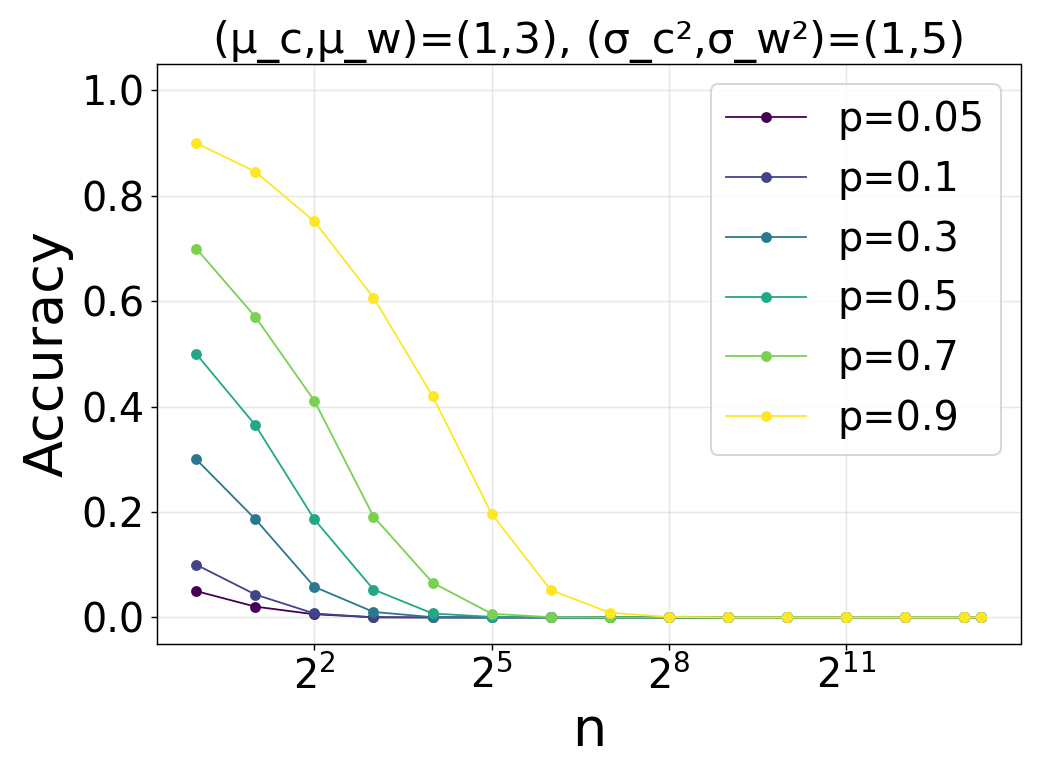}    
  \end{minipage}
  \begin{minipage}{0.24\linewidth}
    \centering
    \includegraphics[width=\linewidth]{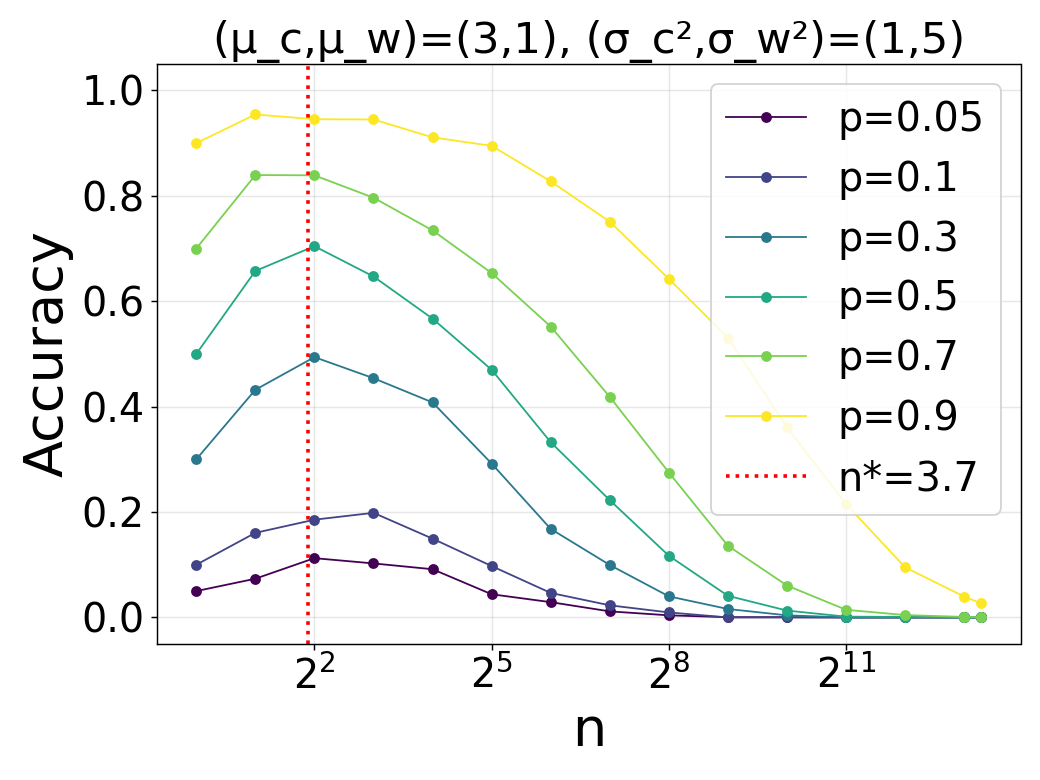}    
  \end{minipage}  
  \begin{minipage}{0.24\linewidth}
    \centering
    \includegraphics[width=\linewidth]{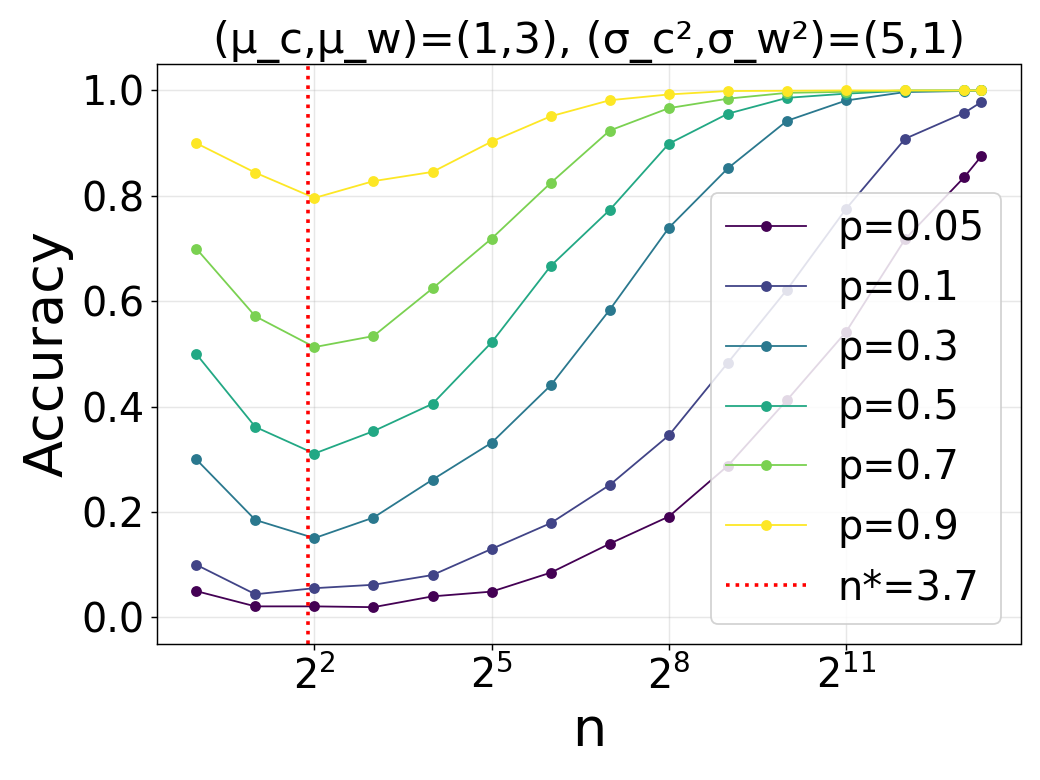}    
  \end{minipage}
  \caption{
    Synthetic Gaussian experiments for the four phases. The observed BoN curves match the predicted iA, iiB, iiA, and iB behaviors, and the predicted $n^\star$ aligns with the empirical optimum in the reward-hacking phase.
  }
  \label{fig:four phase}
\end{figure}

Figure~\ref{fig:scaling theory} shows the results of confirming the consistency between the theoretical formula of the scaling law and the experimental results using synthetic data. The circular points in the figure are the average accuracies from BoN simulations, while the dotted line is the scaling law predicted by the theory, showing that the theory and experiments are consistent. The square points are plotted to confirm Lemma~\ref{prop:typical_region_reduction}, and correspond to the results restricted to cases where $k_n$ correct responses were generated during the simulation, namely $A_n(k_n; x)$. Since the circular points, i.e., $S_n(x)$, and the square points, $A_n(k_n; x)$, take close values, we can confirm the correctness of Lemma~\ref{prop:typical_region_reduction}.
In addition, all results obtained under different simulation conditions are shown in Figure~\ref{fig:scaling theory all}.
%
\begin{wrapfigure}{r}{0.6\linewidth}
\centering
\scalebox{0.8}{
\includegraphics[width=\linewidth]{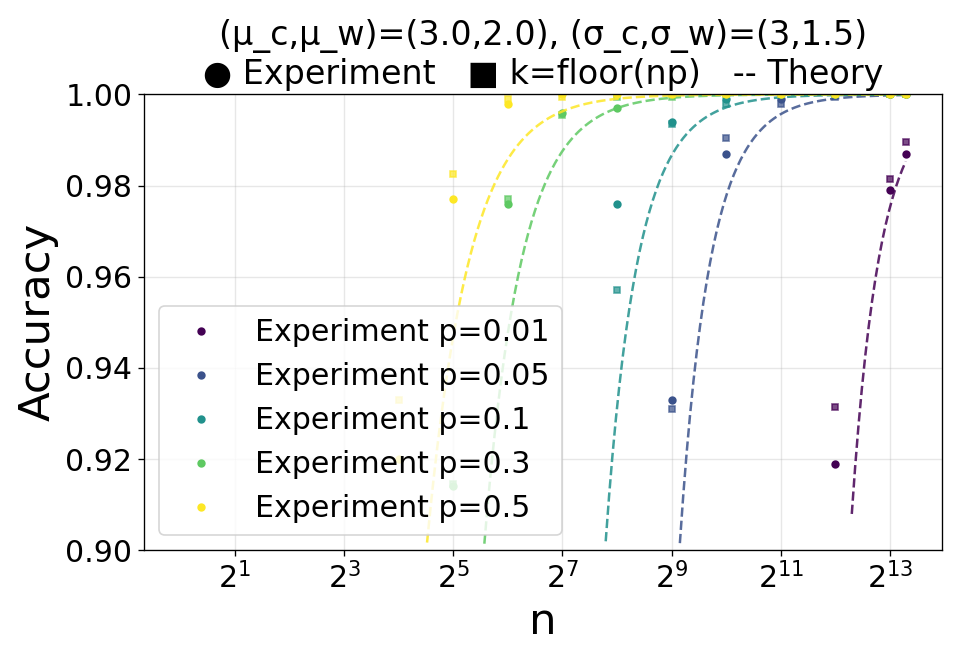}
}
  \caption{
  Experimental validation of the EVT approximated Gaussian scaling law. BoN accuracy follows the predicted finite-$n$ scaling Eq. (\ref{eq:gaussian scaling law}) well. Furthermore, the typical-region accuracy $A_n(k_n;x)$ closely matches $S_n(x)$ as in Lemma~\ref{prop:typical_region_reduction}.
  }
  \label{fig:scaling theory}
\label{fig:scaling theory}
\end{wrapfigure}

\paragraph{Effectiveness of the proposed early stopping method.} We evaluate the effectiveness of the early stopping method introduced in Section~\ref{sec:tail_tension_stopping} in an online setting, where only $n$ generated samples are available for fitting the Gaussian mixture model. 
We show the results of RAMO-Llama3.1-8B, where reward hacking occurs particularly often. Results for other reward models are provided in Appendix~\ref{appendix: full results for tail-tension early stopping}.
Since the extreme-value approximation of the Gaussian distribution has an average deviation of about 3--4\% at $n=32$, we set $n_{\mathrm{th}}=32$. Experimental results with standard deviation are shown in Figure~\ref{experiment:ours}. 
From left to right, the panels show the results on the entire dataset, on datasets where reward hacking occurs, and on datasets where it does not occur. In each case, the upper plot shows accuracy, while the lower plot shows the total number of generated samples. On datasets where reward hacking occurs, the proposed method outperforms standard BoN while also substantially reducing computational cost. In contrast, in the rightmost plot, where accuracy increases monotonically, early stopping is in principle suboptimal; accordingly, the proposed method does not increase accuracy.
\begin{figure}[htb]
  \centering 
  \begin{minipage}{0.3\linewidth}
    \centering
    \includegraphics[width=\linewidth]{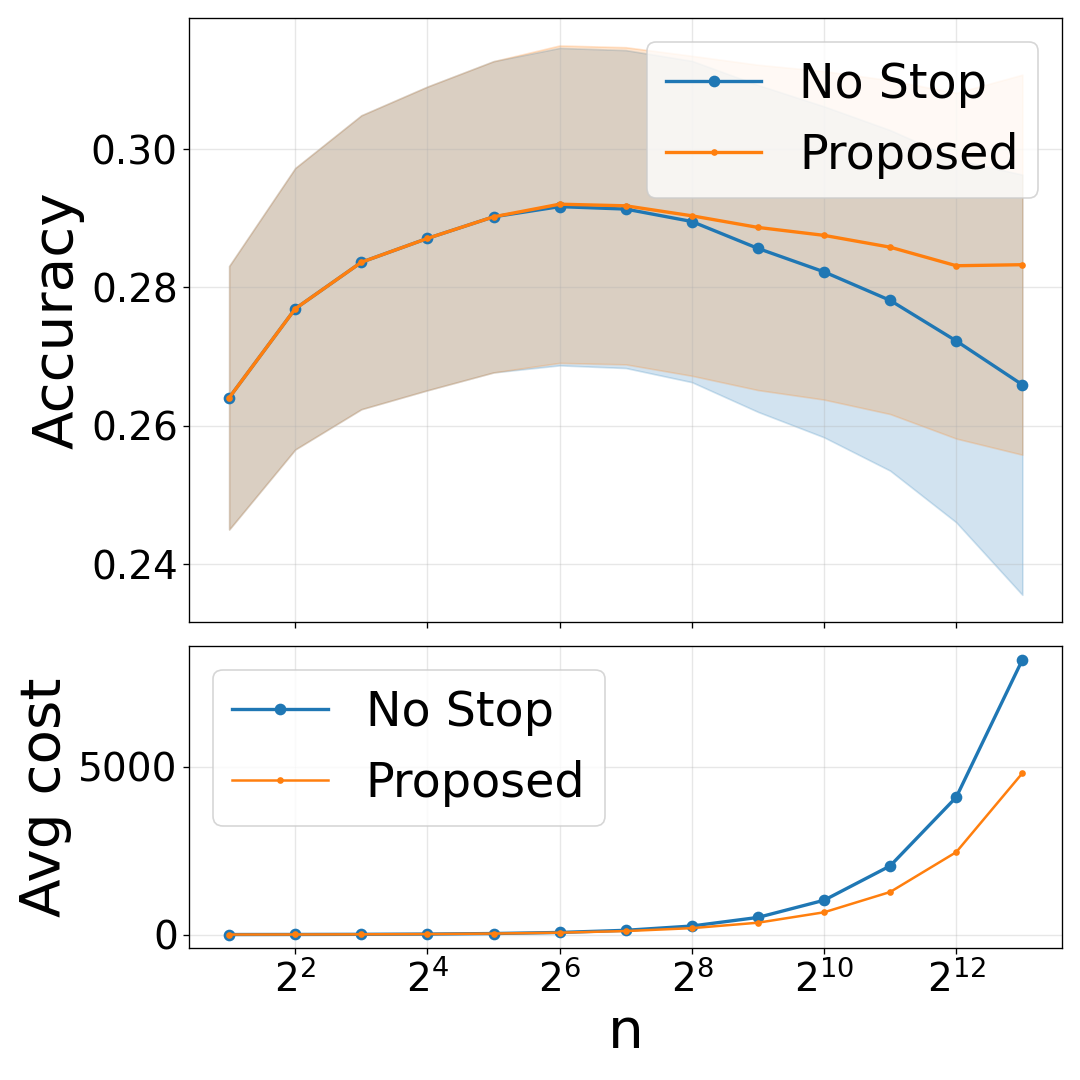}
  \end{minipage}  
  \begin{minipage}{0.3\linewidth}
    \centering
    \includegraphics[width=\linewidth]{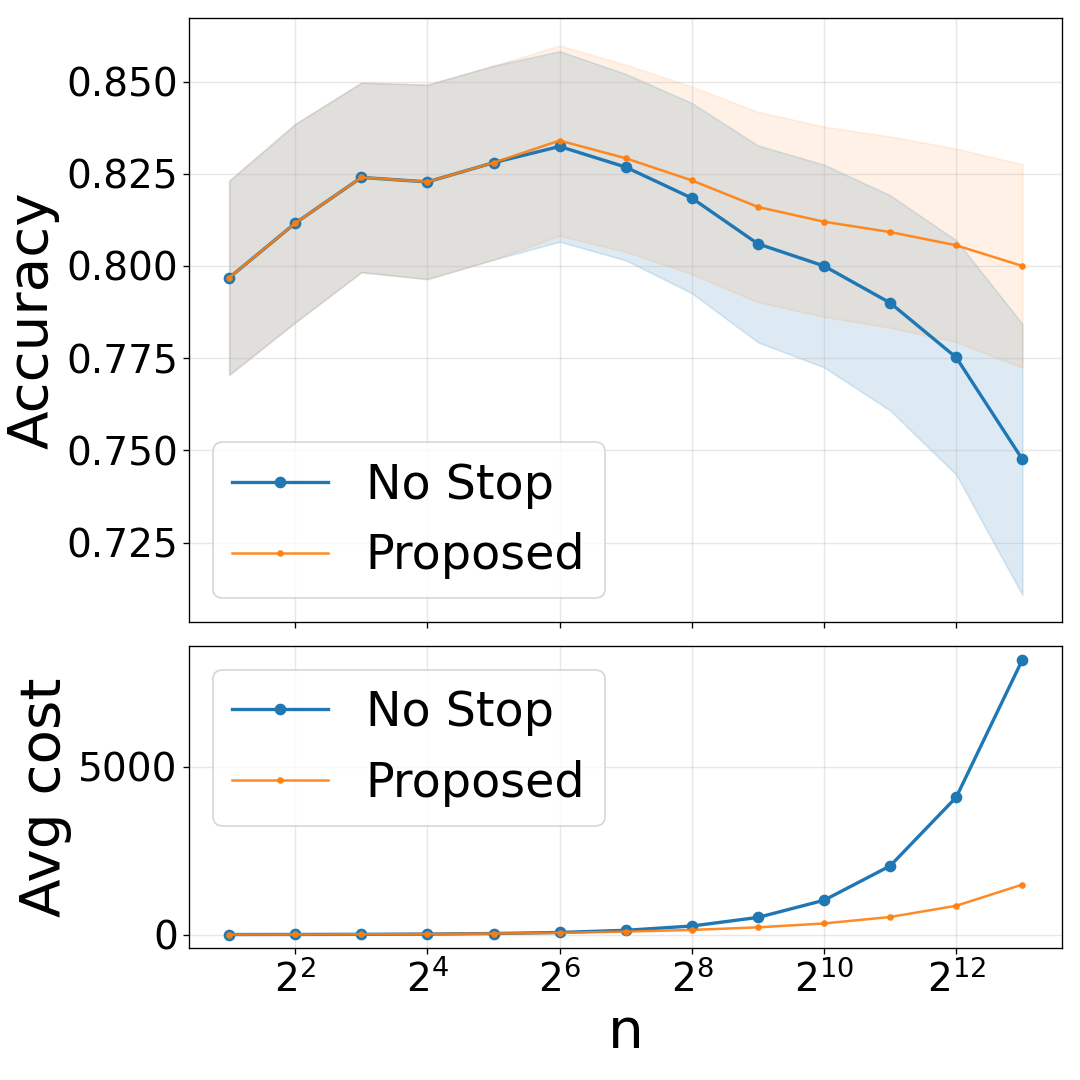}    
  \end{minipage}  
  \begin{minipage}{0.3\linewidth}
    \centering
    \includegraphics[width=\linewidth]{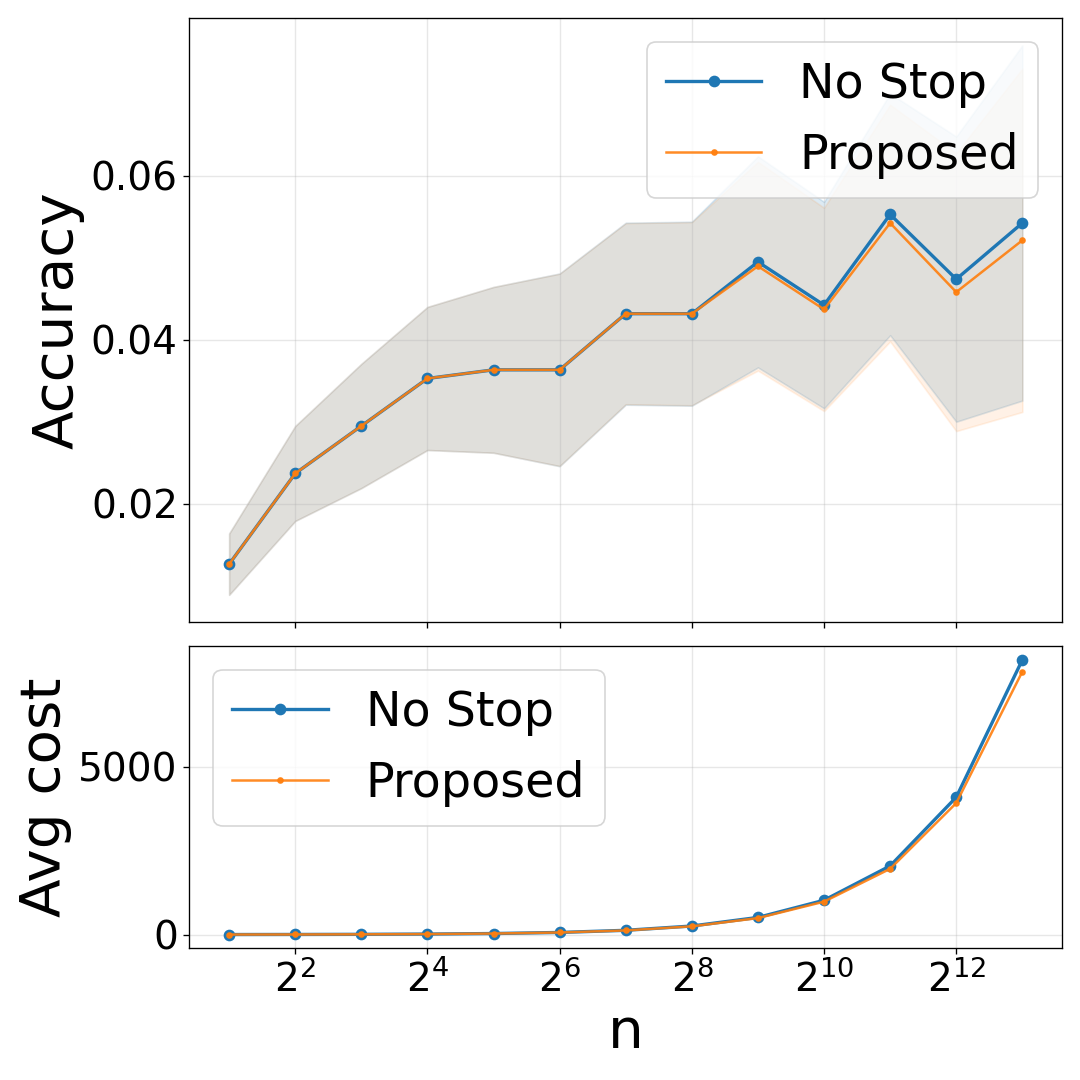}    
  \end{minipage}

  \caption{Tail-tension early stopping on RAMO-Llama3.1-8B scores. The proposed method mitigates degradation in reward-hacking cases while reducing computational cost.
  }
  \label{experiment:ours}
\end{figure}
\section{Related Work} 

\paragraph{Scaling Laws for Inference-Time Scaling Methods.}
Inference-time scaling improves LLM performance by allocating more computation at inference time~\citep{DBLP:journals/corr/abs-2408-03314,DBLP:conf/emnlp/MuennighoffYSLFHZLCH25}. Its basic principle is that, with multiple sampled generations, pass@$n$ approaches one exponentially fast as $n$ increases~\citep{DBLP:journals/corr/abs-2407-21787, DBLP:conf/icml/SchaefferKHJPLJ25}, suggesting that performance can improve if a correct answer can be selected, though this law assumes oracle selection. \cite{NEURIPS2024_51173cf3} showed that Majority Vote and Filter-Vote reduce error roughly exponentially when correctness probability converges to one, e.g., when the correct answer has a positive probability gap over the most frequent incorrect answer. \citet{chen2025provable} proposed knockout-style and league-style selection algorithms with polynomial failure bounds, assuming nonzero correct-generation probability and better-than-random correct--incorrect comparisons.

\paragraph{Analysis of Best-of-N sampling.}
\cite{DBLP:journals/corr/abs-2510-18982} analyze inference-time scaling with a verifier via optimal transport, relating generator coverage, verifier ROC, and sampling sub-optimality. \cite{DBLP:conf/icml/HuangBL0KF25} analyzes BoN regret, explaining overoptimization as a trade-off between gains from larger $n$ and degradation from reward-model errors. Because reward models are only proxies for true objectives such as preferences, safety, or correctness, strong optimization can cause reward hacking or overoptimization~\citep{DBLP:conf/icml/GaoSH23}. \cite{DBLP:journals/corr/abs-2506-19248} formally defines reward hacking and proposes HedgeTune, which tunes dataset-specific parameters using validation data. \cite{DBLP:conf/nips/KwaTG24} argues that reward-model errors dominate in the high-reward regime, especially under heavy tails, and cannot be eliminated by KL regularization. \cite{DBLP:journals/corr/abs-2603-06797} shows that optimistic selection works for light-tailed distributions, whereas heavy-tailed ones require pessimistic regularization, and proposes Best-of-Tails based on Tsallis-divergence regularization. Unlike previous work on KL divergence, regret, or the tail of a single reward distribution, our extreme-value-theoretic analysis characterizes when BoN improves monotonically or suffers reward hacking. In particular, it shows that even with a heavy-tailed marginal reward distribution, BoN can remain in an improving phase when the correct-response tail dominates, yielding new scaling laws and optimal-stopping results.

\section{Conclusion}\label{sec:conclusion}
In this work, we formulated Best-of-N sampling for verifiable problems as a competition between the upper tails of the reward-score distributions for correct and incorrect responses. From this formulation, we derived three main consequences. First, we classified the scaling behavior of BoN into four phases including monotonic scaling and reward hacking. Second, for representative reward-score distributions, we derived finite-sample scaling laws, showing that incomplete reward models can lead to polynomial convergence unlike oracle selection, and that Gaussian reward models can exhibit reward hacking when the variance of incorrect responses is large. Third, as an application of the theory, we derived a tail-tension criterion that estimates the risk of oversampling using only reward scores. Through experiments, we confirmed that the proposed tail-competition perspective explains the phase classification, scaling laws, and emergence of reward hacking in BoN.

\paragraph{Limitations and Future Work. } Our analysis has several limitations, which also suggest promising directions for future work. First, this study focuses on verifiable problems, and applying the proposed theoretical framework to open-ended tasks remains an important direction for future research. Second, while our early stopping method is designed for settings in which the reward distribution separates correct and incorrect answers, further investigation is needed for more general cases, such as non-bimodal or far from Gaussian reward distributions or scenarios in which incorrect answers receive higher average rewards than correct ones. Finally, while our analysis focuses on the asymptotic regime of large $n$ to obtain closed-form scaling laws and characterize the qualitative behavior of BoN, developing practical methods for accurately identifying such regimes at finite $n$ from real-world data remains an important direction for future work.

\bibliographystyle{plainnat}
\bibliography{reference}


\appendix
\section{Proofs}
\newenvironment{proofsketch}
  {
   \noindent\textbf{Proof sketch.}\ }
  {\par\medskip}

\subsection{Lemma~\ref{prop:typical_region_reduction}}\label{appendix:ass:typical_stability proof}

\begin{proofsketch}
We decompose the difference between $S_n$ and $A_n$ according to whether $K$ lies within $M\sqrt{n}$ of $k_n$ or farther away. The contribution from the former term tends to zero by the assumption that the expected variation of $A_n(K;x)$ within $M\sqrt{n}$ of $K=k_n$ is zero. On the other hand, the contribution from the latter term is exponentially small by Hoeffding's inequality.
\end{proofsketch}

\begin{proof}
Let $K\sim\mathrm{Bin}(n,p(x))$. Then
\[
S_n(x)=\mathbb{E}_{K\sim\mathrm{Bin}(n,p(x))}[A_n(K;x)]
\]
holds.
Therefore, for any $M>1$,
\begin{align}\label{eq:appendix:1}
|S_n(x)-A_n(k_n;x)|
&\le
\mathbb{E}\!\left[
|A_n(K;x)-A_n(k_n;x)|
\mathbf{1}\{|K-k_n|\le M\sqrt n\}
\right]\nonumber
\\
&\qquad
+\,\Pr(|K-k_n|>M\sqrt n)
\end{align}
Here, we used the fact that the absolute value of the expectation is no greater than the expectation of the absolute value, and we split the possible values of $K$ into the region near the mean and the region that deviates substantially from the mean; the latter term is bounded by using $0 \leq A_n\leq 1$, which implies $|A_n(K;x) - A_n(k_n;x)|\leq 1$.

Since $0\leq np(x)-k_n< 1$, if $|K-k_n| > M\sqrt{n}$ and $n \geq 1 $, then by the triangle inequality and the assumption, $|K-np(x)| \geq |K-k_n| - |np(x)-k_n| > M\sqrt{n}-1 \geq M\sqrt{n}-\sqrt{n}$ holds. Therefore, by the monotonicity of probability,

\[
\Pr(|K-k_n|>M\sqrt n)
\le
\Pr\!\left(|K-np(x)|>(M-1)\sqrt n\right).
\]

By using Hoeffding's inequality, we have
\[
\Pr\!\left(|K-np(x)|>(M-1)\sqrt n\right)
=
\Pr\!\left(
\left|\frac{K}{n}-p(x)\right|>\frac{M-1}{\sqrt n}
\right)
\le
2e^{-2(M-1)^2}.
\]
Here, the coefficient $2$ appears because Hoeffding's inequality is applied to each of the two tails in
$\Pr\!\left(\left|\frac{K}{n}-p(x)\right|>\frac{M-1}{\sqrt n}\right) = \Pr\!\left(
\frac{K}{n}-p(x)>\frac{M-1}{\sqrt n}
\right) + \Pr\!\left(
\frac{K}{n}-p(x)<-\frac{M-1}{\sqrt n}
\right)$.

Finally, the following holds.
\[
\limsup_{n\to\infty}|S_n(x)-A_n(k_n;x)|
\le
2e^{-2(M-1)^2}.
\]
Here, we used the assumption of the lemma to conclude that the first term on the right-hand side of Equation~\ref{eq:appendix:1} converges to $0$ as $n\rightarrow\infty$.
\end{proof}

\subsection{Theorem~\ref{theorem:generalized_three_regimes}}
\begin{proof}[Proof]\label{proof:generalized_three_regimes}
First, define the expression inside the probability in $A_n$ by
\begin{eqnarray}
    W_n := \alpha_n Z_c -Z_w + \beta_n,
\end{eqnarray}
and compute its limit as $n\rightarrow\infty$. Then, computing the final probability $\lim_{n\rightarrow \infty}A_n(k_n;x) = \lim_{n\rightarrow \infty}\Pr(W_n > 0)$ gives Table~\ref{tab:appendix:1} below. 

\begin{table}[htb]
\caption{Limits of $A_n(k_n;x)$ classified by the limiting values of $\alpha_n$ and $\beta_n$ as $n\rightarrow \infty$}\label{tab:appendix:1}
\centering
\begin{tabular}{cccc}
\toprule
Limit of $\alpha_n$ & Other limit & Limit of $W_n$ & Limit of $A_n(k_n;x)$ \\
\midrule
$\alpha_n \to \alpha \in (0,\infty)$
& $\beta_n \to \infty$
& $\infty$
& $1$ \\
$\alpha_n \to \alpha \in (0,\infty)$
& $\beta_n \to -\infty$
& $-\infty$
& $0$ \\
$\alpha_n \to \alpha \in (0,\infty)$
& $\beta_n \to \beta \in \mathbbm{R}$
& $\alpha Z_c + \beta - Z_w$
& $\Pr(\alpha Z_c + \beta - Z_w > 0)$\\
\midrule
$\alpha_n \to \infty$
& $\dfrac{\beta_n}{\alpha_n} \to \infty$
& $\infty$
& $1$ \\
$\alpha_n \to \infty$
& $\dfrac{\beta_n}{\alpha_n} \to -\infty$
& $-\infty$
& $0$ \\
$\alpha_n \to \infty$
& $\dfrac{\beta_n}{\alpha_n} \to \tau \in \mathbbm{R}$
& $\frac{W_n}{\alpha_n} \rightarrow (Z_c + \tau)$
& $\Pr(Z_c > -\tau)$\\
\midrule
$\alpha_n \to 0$
& $\beta_n \to \infty$
& $\infty$
& $1$ \\
$\alpha_n \to 0$
& $\beta_n \to -\infty$
& $-\infty$
& $0$ \\
$\alpha_n \to 0$
& $\beta_n \to \beta \in \mathbbm{R}$
& $\beta - Z_w$
& $Pr(Z_w < \beta)$\\
\bottomrule
\end{tabular}
\end{table}

\end{proof}



\subsection{Lemma~\ref{prop:evt_approximation_error_bound}}\label{appendix:evt_approximation_error_bound}
\begin{proof}
Let $\tilde M_{c,n}:=a_c(k_n)Z_c + b_c(k_n), \tilde M_{w,n}:=a_w(m_n)Z_w + b_w(m_n)$, and $F_{c, n}, F_{w, n}, \tilde F_{c, n}$ and $\tilde F_{w, n}$ be the CDF for $M_c, M_w, \tilde M_{c,n}$ and $\tilde M_{w,n}$, respectively. Since $\Pr(M_c>M_w)=\mathbb{E}_{M_c}[F_{w,n}(M_c)]$ and
$\Pr(\tilde M_{c,n}>\tilde M_{w,n})
=\mathbb{E}_{\tilde M_{c,n}}[\tilde F_{w,n}(\tilde M_{c,n})]$, by the triangle inequality, the difference is bounded by
$\varepsilon_{w,n} + \left|
\mathbb{E}_{M_c}[\tilde F_{w,n}(M_c)]
-
\mathbb{E}_{\tilde M_{c,n}}[\tilde F_{w,n}(\tilde M_{c,n})]
\right|$. Since $\tilde F_{w,n}$ is a continuous increasing CDF
with values in $[0,1]$, the tail integral formula gives, for any random
variable $X$,
\[
\mathbb{E}[\tilde F_{w,n}(X)]
=
\int_0^1 \Pr\!\left(X>\tilde F_{w,n}^{-1}(u)\right)\,du.
\]
Therefore,
\[
\begin{aligned}
\left|
\mathbb{E}[\tilde F_{w,n}(M_c)]
-
\mathbb{E}[\tilde F_{w,n}(\tilde M_{c,n})]
\right|
&\le
\int_0^1
\left|
F_{c,n}(\tilde F_{w,n}^{-1}(u))
-
\tilde F_{c,n}(\tilde F_{w,n}^{-1}(u))
\right|du \\
&\le \varepsilon_{c,n}.
\end{aligned}
\]
Combining the two bounds proves the claim.
\end{proof}

\subsection{Theorem~\ref{theorem:gumbel scaling}}\label{appendix:proof:gumbel}
\begin{proofsketch}
Since $\Pr\{G_w\ge \alpha_nG_c+\beta_n\}=1-\exp\{-\mathrm{e}^{-(\alpha_nG_c+\beta_n})\}\rightarrow
\mathrm{e}^{-\beta_n}\mathrm{e}^{-\alpha_n G_c},$ and \(\alpha_n\to\alpha\) and \(\beta_n\to\infty\), the probability is expected to decay to zero approximately at the rate \(\mathrm{e}^{-\beta_n}\). To determine the coefficient, we apply the dominated convergence theorem to $
\mathrm{e}^{\beta_n}\Pr\{G_w\ge \alpha_nG_c+\beta_n\},$ which allows us to interchange the limit and the expectation and thereby reduce the problem to evaluating the integral $\mathbbm{E}_{G_c}\!\left[\mathrm{e}^{-\alpha G_c}\right].$
\end{proofsketch}

\begin{proof}
With Gumbel extreme-value approximation
$\psi(t):=\Pr(G_w\ge t)=1-\exp(-e^{-t})$, independence gives
$\Pr\{G_w\ge \alpha_nG_c+\beta_n\}=\mathbb{E}_{G_c}\bigl[\psi(\alpha_nG_c+\beta_n)\bigr]$.
Define $H_n(y):=e^{\beta_n}\psi(\alpha_n y+\beta_n)$.  Since
$\beta_n\to\infty$ and $\alpha_n\to\alpha<\infty$, for each fixed $y$, $H_n(y)=e^{\beta_n}\left(1-\exp\{-e^{-(\alpha_ny+\beta_n)}\}\right)\to e^{-\alpha y}$.
Moreover, using $1-e^{-u}\le u$, we have $0\le H_n(y)\le e^{-\alpha_n y}$.  For sufficiently large $n$, choose $R>0$ such that $0\le\alpha_n\le R$; then
$H_n(y)\le 1+e^{-Ry}$.  This is integrable under the standard Gumbel law, because for $q>-1$,
$\mathbb{E}[e^{-qG}]=\Gamma(1+q)$, obtained from the change of variables $u=e^{-t}$ in the Gumbel density.
Hence the dominated convergence theorem yields
\[
e^{\beta_n}
\Pr\{G_w\ge \alpha_nG_c+\beta_n\}
=
\mathbb{E}[H_n(G_c)]
\to
\mathbb{E}[e^{-\alpha G_c}]
=
\Gamma(1+\alpha).
\]
Therefore,
$\Pr\{G_w\ge \alpha_nG_c+\beta_n\}=\Gamma(1+\alpha)e^{-\beta_n}+o(e^{-\beta_n})$ holds, and we obtain
\[
\tilde{A}_n(k_n;x)
=
1-
\Gamma(1+\alpha)e^{-\beta_n(x)}
+
o\!\left(e^{-\beta_n(x)}\right).
\]
\end{proof}

\subsection{Theorem~\ref{theorem:frechet scaling}}\label{appendix:proof:frechet}
\begin{proofsketch}
The basic idea is the same as in the proof for the Gumbel distribution (Appendix~\ref{appendix:proof:gumbel}): instead of dividing by $\mathrm{e}^{-\beta_n}$, we divide the probability by  $\alpha_n^{-\rho_w}$, which converges to zero to calculate the coefficient.
\end{proofsketch}
\begin{proof}
With the Fr\'echet extreme-value approximation,
$Z_c$ and $Z_w$ are independent and have Fr\'echet distribution function $\exp(-z^{-\rho_c})$ and $\exp(-z^{-\rho_w})$, respectively. 
Set $r_n:=\alpha_n^{-\rho_w}$.  In the correct-side extreme-value-dominant phase, $\alpha_n\to\infty$, hence $r_n\to0$.  Therefore,$1-\Pr\{\alpha_n Z_c>Z_w\}=\mathbb{E}\left[1-\exp\left\{-r_nZ_c^{-\rho_w}\right\}\right]$.
Using $0\le(1-e^{-u})\le u$ for $u>0$, the integrand divided by $r_n$ converges pointwise to $Z_c^{-\rho_w}$ and is dominated by $Z_c^{-\rho_w}$.  This dominating variable is integrable: if $T:=Z_c^{-\rho_c}$, then $T\sim\operatorname{Exp}(1)$, and hence
\[
\mathbb{E}[Z_c^{-\rho_w}]
=
\mathbb{E}\left[T^{\rho_w/\rho_c}\right]
=
\Gamma\!\left(1+\frac{\rho_w}{\rho_c}\right)
<
\infty.
\]
By dominated convergence,
$1-\Pr\{\alpha_n Z_c>Z_w\}=\Gamma\!\left(1+\frac{\rho_w}{\rho_c}\right)\alpha_n^{-\rho_w}+o\!\left(\alpha_n^{-\rho_w}\right)$.
Therefore,
\[
\tilde{A}_n(k_n;x)
=
1-
\Gamma\!\left(1+\frac{\rho_w}{\rho_c}\right)
\alpha_n(x)^{-\rho_w}
+
o\!\left(\alpha_n(x)^{-\rho_w}\right).
\]
\end{proof}

\subsection{Theorem~\ref{theorem:weibull scaling}}\label{appendix:proof:weibull}
\begin{proofsketch}
The basic idea is the same as in the proof for the Gumbel distribution (Appendix~\ref{appendix:proof:gumbel}): instead of dividing by $\mathrm{e}^{-\beta_n}$, we divide the probability by  $\alpha_n^{-\frac{1}{\xi_w}}$, which converges to zero and examine the coefficient.
\end{proofsketch}

\begin{proof}
Since the common upper endpoint is $r^{\max}$, the event
$M_c(k_n;x)>M_w(m_n;x)$ is equivalent to
$r^{\max}-M_c(k_n;x)<r^{\max}-M_w(m_n;x)$.  
Since $b_c = b_w$ from the assumption, $\beta_n = 0$. Therefore, using the Weibull survival function $\Pr(Y_w>y)=\exp(-y^{-\frac{1}{\xi_w}})$,
\[
\Pr\{\alpha_n Y_c<Y_w\}
=
\mathbb{E}_{Y_c}\left[
\exp\left\{-\alpha_n^{-\frac{1}{\xi_w}}Y_c^{-\frac{1}{\xi_w}}\right\}
\right].
\]
Set $r_n:=\alpha_n^{-\frac{1}{\xi_w}}$.  In the correct-side extreme-value-dominant phase, $\alpha_n\to0$, hence $r_n\to0$. 
Again, $0\le(1-e^{-u})/u\le1$ implies that the integrand divided by $r_n$ converges pointwise to $Y_c^{-\frac{1}{\xi_w}}$ and is dominated by $Y_c^{-\frac{1}{\xi_w}}$.  This variable is integrable: if $T:=Y_c^{-\frac{1}{\xi_c}}$, then $T\sim\operatorname{Exp}(1)$, so
\[
\mathbb{E}[Y_c^{-\frac{1}{\xi_w}}]
=
\mathbb{E}\left[T^{-\frac{1}{\xi_w}/-\frac{1}{\xi_c}}\right]
=
\Gamma\!\left(1+\frac{-\frac{1}{\xi_w}}{-\frac{1}{\xi_c}}\right)
<
\infty.
\]
Therefore,$1-\Pr\{\alpha_nY_c<Y_w\}=\Gamma\!\left(1+\frac{-\frac{1}{\xi_w}}{-\frac{1}{\xi_c}}\right)
\alpha_n^{-\frac{1}{\xi_w}}+o\!\left(\alpha_n^{-\frac{1}{\xi_w}}\right)$.
Thus,
\[
\tilde{A}_n(k_n;x)
=
1-
\Gamma\!\left(1+\frac{-\frac{1}{\xi_w}}{-\frac{1}{\xi_c}}\right)
\alpha_n(x)^{-\frac{1}{\xi_w}}
+
o\!\left(\alpha_n(x)^{-\frac{1}{\xi_w}}\right).
\]
\end{proof}

\subsection{Scaling Law for Gaussian Distribution without Extreme Value Approximation}\label{appendix:precise gaussian scaling}
We calculate the scaling law for reward model with Gaussian distribution without extreme value approximation. 

Let $Z_k=\frac{M_c(k)-\mu_c}{\sigma_c}$. Then, 
\begin{eqnarray}
  1-A_n(k) = 1 - \mathbb{E}_{Z_k}[\Phi(\frac{\sigma_c Z_k + \mu_c - \mu_w}{\sigma_w})^m],
\end{eqnarray}
since $1-A_n(k)=\Pr(M_c(k) < M_w(m)) = \Pr(\sigma_{c}Z_k+\mu_c < M_w)=1-\Pr(\sigma_cZ_k + \mu_c \geq M_w) = 1 - \mathbbm{E}_{Z_k}[\Pr(\sigma_cz + \mu_c \geq R_w)^m] = 1 - \mathbbm{E}_{Z_k}[\Pr(\frac{\sigma_c Z_k + \mu_c - \mu_w}{\sigma_w} \geq \frac{\sigma_c R_w - \mu_w}{\sigma_w})^m] = 1 - \mathbb{E}_{Z_k}[\Phi(\frac{\sigma_c Z_k + \mu_c - \mu_w}{\sigma_w})^m]$, where $\Phi$ is the CDF for standard normal distribution.

First, to approximate the integrand for large $k$, we calculate the Mills' ratio for $1-\Phi(x)$. The result is 
\begin{eqnarray}
  1-\Phi(x) = \frac{\phi(x)}{x}(1 + O(x^{-2})),\label{eq:mills ratio}
\end{eqnarray}
where $\phi$ is PDF for standard normal distribution. To confirm Eq.\eqref{eq:mills ratio}, $1-\Phi(x) = \int_x^{\infty}\phi(u)du = -\int_x^{\infty} \frac{\phi^{\prime}(u)}{u} du = -[\frac{\phi(u)}{u}]_x^{\infty} - \int_x^{\infty} \frac{\phi(u)}{u^2}du = \frac{\phi(x)}{x} - \int_x^{\infty}\frac{\phi(u)}{u^2}du$. We used $\phi^{\prime}(u) = -u\phi(u)$ for the standard normal distribution for the second equality.
Since $\int_x^{\infty}\frac{\phi(u)}{u^2}du \leq \frac{1}{x^2}\int_x^{\infty}\phi(u)du = \frac{1-\Phi(x)}{x^2}$, 
we get $1-\Phi(x) \geq \frac{\phi(x)}{x} - \frac{1-\Phi(x)}{x^2} \Leftrightarrow 1-\Phi(x) \geq \frac{\phi(x)}{x}(1+\frac{1}{x^2})^{-1} = \frac{\phi(x)}{x}(1-\frac{1}{x^2} + O(x^{-4}))$. Combining $1-\Phi(x) \leq \frac{\phi(x)}{x}$, Eq.\eqref{eq:mills ratio} follows. 

Second, we prove that $u_k:=\Phi^{-1}(1-\frac{1}{k})$ for large $k$ satisfies the following.
\begin{eqnarray}
  u_k^2 = 2\log{k} - \log{4\pi} -\log{\log{k}} + o(1). \label{eq:u_k^2}
\end{eqnarray}
By the definition of $u_k$ and Mills' ratio, we get $1-\Phi(u_k) = \frac{1}{k}\Leftrightarrow \frac{\phi(u_k)}{u_k}(1 + O(u_k^{-2})) = \frac{1}{k} \Leftrightarrow u_k^2 + 2\log{\sqrt{2\pi}} + 2\log{u_k} + O(u_k^{-2}) = 2\log{k}$, where we used $\log{(1+x)} = x + O(x^{2})$ for $x\rightarrow 0$. Then, comparing leading order, $u_k^2(1 + o(1))=2\log{k} \Leftrightarrow u_k^2 = 2\log{k}(1 + o(1))$. Returning back to the original equation, we get $u_k^2 = 2\log{k} - \log{4\pi} -\log{\log{k}} + o(1)$.

Additionally, convenient formulas as: 
\begin{eqnarray}
  &&\frac{u_k}{u_{\frac{k}{v}}} = \sqrt{\frac{2\log{k}+o(\log{k})}{2\log{\frac{k}{v}}+o(\log{k})}} \rightarrow 1, \\
  &&u_{\frac{k}{v}}^2 - u_k^2= 
  2\log{k} - 2\log{v} - \log{4\pi} -\log\{\log{k(1 - \frac{\log{v}}{\log{k}})}\} - (2\log{k} - \log{4\pi} -\log\log{k}),\nonumber\\
  &&\hspace{10cm}\rightarrow -2\log{v}
  \\
  &&u_{\frac{k}{v}} - u_k = \frac{u_{\frac{k}{v}}^2 - u_k^2}{u_{\frac{k}{v}} +  u_k} \rightarrow 0.
\end{eqnarray}

Then, we define $v = k(1 - \Phi(z))\Leftrightarrow z=u_{\frac{k}{v}}$, and $\frac{dz}{dv}=-\frac{1}{k\phi(z)}$ holds. Thus, 
\begin{eqnarray}
1-A_n(k)&=& \int_{-\infty}^{\infty}[1 - \Phi^m(\frac{\sigma_c}{\sigma_w}z + \frac{\mu_c - \mu_w}{\sigma_w})]k\phi(z)\Phi(z)^{k-1}dz \\
&=& \int_0^k[1 - \Phi^m(\frac{\sigma_c}{\sigma_w}u_{\frac{k}{v}}+\frac{\mu_c-\mu_w}{\sigma_w})](1 - \frac{v}{k})^{k-1}dv.  
\end{eqnarray}

Dividing the integral by $m(1 - \Phi(\frac{\sigma_c}{\sigma_w}u_k + \frac{\mu_c - \mu_w}{\sigma_w}))$, 
\begin{eqnarray}
&&\frac{1 - A_n(k)}{m(1 - \Phi(\frac{\sigma_c}{\sigma_w}u_k + \frac{\mu_c - \mu_w}{\sigma_w}))} \nonumber\\
&&= 
\int_0^k\frac{[1 - \Phi^m(\frac{\sigma_c}{\sigma_w}u_{\frac{k}{v}}+\frac{\mu_c-\mu_w}{\sigma_w})]}{m(1 - \Phi(\frac{\sigma_c}{\sigma_w}u_k + \frac{\mu_c - \mu_w}{\sigma_w}))}(1 - \frac{v}{k})^{k-1}dv\nonumber\\
&&=\int_0^k\frac{[m(1 - \Phi(\frac{\sigma_c}{\sigma_w}u_{\frac{k}{v}}+\frac{\mu_c-\mu_w}{\sigma_w}) + O(m^2(1-\Phi)^2)]}{m(1 - \Phi(\frac{\sigma_c}{\sigma_w}u_k + \frac{\mu_c - \mu_w}{\sigma_w}))}(1 - \frac{v}{k})^{k-1}dv. \nonumber \\ 
\end{eqnarray}
By using Mills ratio, 
\begin{eqnarray}
&&\frac{1 - \Phi(\frac{\sigma_c}{\sigma_w}u_{\frac{k}{v}}+\frac{\mu_c-\mu_w}{\sigma_w})}{1 - \Phi(\frac{\sigma_c}{\sigma_w}u_k + \frac{\mu_c - \mu_w}{\sigma_w})} = \frac{\frac{\sigma_c}{\sigma_w}u_k+\frac{\mu_c-\mu_w}{\sigma_w}}{\frac{\sigma_c}{\sigma_w}u_{\frac{k}{v}}+\frac{\mu_c-\mu_w}{\sigma_w}}\nonumber\\
&&\hspace{1cm}\cdot\exp\left[-\frac{(\frac{\sigma_c}{\sigma_w}u_{\frac{k}{v}}+\frac{\mu_c-\mu_w}{\sigma_w})^2 - (\frac{\sigma_c}{\sigma_w}u_k+\frac{\mu_c-\mu_w}{\sigma_w})^2}{2}\right](1+o(1)).\nonumber\\
\end{eqnarray}
Exponential becomes $-\frac{1}{2}\left[(\frac{\sigma_c}{\sigma_w})^2 (u_{\frac{k}{v}}^2 - u_k^2)+2\frac{\sigma_c(\mu_c - \mu_w)}{\sigma_w^2}(u_{\frac{k}{v}} - u_k)\right]$. Therefore, applying formulas for $u_k$, we obtain
\begin{eqnarray}
  \frac{1 - \Phi(\frac{\sigma_c}{\sigma_w}u_{\frac{k}{v}}+\frac{\mu_c-\mu_w}{\sigma_w})}{1 - \Phi(\frac{\sigma_c}{\sigma_w}u_k + \frac{\mu_c - \mu_w}{\sigma_w})} \rightarrow 1\cdot \exp\left[\frac{(\frac{\sigma_c}{\sigma_w})^2\cdot (-2\log{v}) - 0}{2}\right] = v^{r^2}.
\end{eqnarray}

Finally, we can get 
\begin{eqnarray}
  \frac{1 - A_n(k)}{m(1 - \Phi(\frac{\sigma_c}{\sigma_w}u_k + \frac{\mu_c - \mu_w}{\sigma_w}))} \rightarrow \int_0^{\infty}v^{(\frac{\sigma_c}{\sigma_w})^2}\mathrm{e}^{^v} dv = \Gamma\left(1 + \left(\frac{\sigma_c}{\sigma_w}\right)^2\right).  
\end{eqnarray}

Thus
\[
1-A_n(k_n)
\sim
\Gamma(1+r^2)\,
m_n\left(1-\Phi\left(\frac{\sigma_c}{\sigma_w} \Phi^{-1}(1 - \frac{1}{k_n})+\frac{\mu_c-\mu_w}{\sigma_w}\right)\right).
\]

Furthermore, using Eq.~\eqref{eq:u_k^2} and Mills' ratio, we obtain the following 
scaling law for Gaussian distribution without extreme value approximation.
\[
\begin{aligned}
1-A_n(k_n)
\sim
\frac{\Gamma(1+(\frac{\sigma_c}{\sigma_w})^2)}{\frac{\sigma_c}{\sigma_w}}
\frac{1-p}{p^{(\frac{\sigma_c}{\sigma_w})^2}}
(4\pi\log n)^{((\frac{\sigma_c}{\sigma_w})^2-1)/2}
n^{-((\frac{\sigma_c}{\sigma_w})^2-1)}\nonumber\\
\hspace{5mm}\cdot\exp
\!\left(-\frac{1}{2\sigma_w^2}\left\{
2\sigma_c(\mu_c - \mu_w)\sqrt{2\log n}
+(\mu_c - \mu_w)^2 \right\}
\right).
\end{aligned}
\]

\section{Supplementary Experimental Results}
\subsection{Misassignment in Tail-Tension Early Stopping}~\label{appendix:misassignment}
We introduced tail-tension early stopping method of Best-of-N sampling in Section~\ref{sec:tail_tension_stopping}. In this section, we discuss the failure mode of our method, specifically the cases in which the heuristic ``higher-mean component $\leftrightarrow$ correct-answer component'' assignment fails. We show that, such cases are empirically rare in our experiments, and cases in which this mismatch actually degrades the performance of the early-stopping method are even rarer. Moreover, this risk can be further mitigated by adopting a stricter early-stopping criterion.

For simplicity, consider the case in which the reward distributions are well approximated by Gaussian distributions. When $\mu_w > \mu_c$, the accuracy of BoN exhibits one of two behaviors as $n$ increases: it either decreases monotonically, or first decreases and then recovers for large $n$, resulting in a U-shaped curve. 

\noindent
\textbf{Phase iiB: Monotonic decrease.}

When $\mu_c < \mu_w$ and $\sigma_c < \sigma_w$, the components are
misaligned under the proposed method, resulting in
$\sigma_l < \sigma_h$. Early stopping is therefore not triggered, and
the method behaves identically to standard BoN.

\noindent
\textbf{Phase iB: Delayed improvement (e.g., a U-shaped curve).}

When $\mu_c < \mu_w$ and $\sigma_c > \sigma_w$, the components are again
misaligned under the proposed method, but now $\sigma_l > \sigma_h$.
Early stopping is therefore triggered and may prevent the method from
benefiting from the recovery in performance that occurs at large $n$.

We show that the phases iiB and iB are rare in the real-world reward models. 
The table~\ref{tab:misassignment ratio} reports the fraction of problem instances for which the
empirical mean reward of incorrect responses $\mu_w$ exceeds that of correct
responses $\mu_c$ by at least $t$ (phases iiB and iB), and vice versa (phases iiA and iA). These results show that misaligned cases are far less frequent than correctly aligned cases, supporting the empirical validity of this assumption in the settings we evaluate.

\begin{center}
\begin{tabular}{ccc}
\toprule
$t$ & $\mu_c \leq \mu_w - t$ & $\mu_c \geq \mu_w + t$ \\
\midrule
0.0 & 11.11\% & 88.89\% \\
0.5 & 5.99\%  & 82.43\% \\
1.0 & 3.49\%  & 73.41\% \\
2.0 & 1.35\%  & 54.61\% \\
\bottomrule
\end{tabular}\label{tab:misassignment ratio}
\end{center}

The table~\ref{tab:misassignment degradation ratio} further reports the proportion of cases satisfying $(\mu_c \leq \mu_w - t)$ and $(\sigma_c > \sigma_w)$, phase iiB, which early stopping can be harmful, as well as the proportion satisfying $(\mu_c \geq \mu_w + t)$ and $(\sigma_w > \sigma_c)$, phase iiA, for which early stopping is beneficial.

\begin{center}
\begin{tabular}{ccc}
\toprule
$t$
& $\mu_c \leq \mu_w - t$ and $\sigma_c > \sigma_w$
& $\mu_c \geq \mu_w + t$ and $\sigma_w > \sigma_c$ \\
\midrule
0.0 & 1.39\% & 70.76\% \\
0.5 & 0.60\% & 65.14\% \\
1.0 & 0.40\% & 57.19\% \\
2.0 & 0.19\% & 40.42\% \\
\bottomrule
\end{tabular}\label{tab:misassignment degradation ratio}
\end{center}

By triggering early stopping only when the difference between the mean
values is at least $1.0$, we can retain the ability to stop 57.19\% of
 possible hacking cases while reducing the fraction of instances in the
potentially harmful reversal regime to 0.40\%.


\subsection{Full Results for Tail-Tension Early Stopping}\label{appendix: full results for tail-tension early stopping}
We show the full results using all reward models excpet for RAMO-Llama3.1-8B, which was used in Section~\ref{sec:experimental results}. The results show that, in many cases, inference cost can be reduced with little or no loss in overall accuracy. In particular, among reward-hacking cases, our method suppresses the degradation in accuracy relative to standard BoN.
\begin{table*}[htb]
\centering
\small
\begin{tabular}{lccccccccc}
\toprule
\textbf{Method / $n$}
& $2^5$ & $2^6$ & $2^7$ & $2^8$ & $2^9$
& $2^{10}$ & $2^{11}$ & $2^{12}$ & $2^{13}$ \\
\midrule

\multicolumn{10}{l}{\textbf{Skywork-Reward-v2-Qwen3-1.7B}} \\
No-stop (all)
& 0.3260 & 0.3339 & 0.3427 & 0.3506 & 0.3546
& 0.3605 & 0.3664 & 0.3712 & 0.3731 \\
Early stopping (all)
& 0.3260 & 0.3336 & 0.3414 & 0.3484 & 0.3527
& 0.3582 & 0.3643 & 0.3693 & 0.3729 \\
No-stop (hack)
& 0.1576 & 0.1572 & 0.1609 & 0.1626 & 0.1576
& 0.1507 & 0.1330 & 0.1033 & 0.0366 \\
Early stopping (hack)
& 0.1576 & 0.1569 & 0.1606 & 0.1624 & 0.1593
& 0.1563 & 0.1462 & 0.1255 & 0.0792 \\
Cost saved
& 2.8\% & 7.4\% & 12.6\% & 17.4\% & 21.0\%
& 23.6\% & 25.4\% & 26.7\% & 27.6\% \\

\midrule
\multicolumn{10}{l}{\textbf{Skywork-Reward-v2-Qwen3-4B}} \\
No-stop (all)
& 0.3313 & 0.3396 & 0.3468 & 0.3539 & 0.3607
& 0.3652 & 0.3707 & 0.3739 & 0.3751 \\
Early stopping (all)
& 0.3313 & 0.3391 & 0.3464 & 0.3532 & 0.3587
& 0.3639 & 0.3696 & 0.3727 & 0.3756 \\
No-stop (hack)
& 0.2518 & 0.2532 & 0.2567 & 0.2589 & 0.2613
& 0.2555 & 0.2503 & 0.2301 & 0.1785 \\
Early stopping (hack)
& 0.2518 & 0.2526 & 0.2567 & 0.2593 & 0.2612
& 0.2602 & 0.2579 & 0.2402 & 0.1999 \\
Cost saved
& 2.0\% & 5.3\% & 9.0\% & 12.6\% & 15.4\%
& 17.5\% & 19.1\% & 20.2\% & 21.0\% \\

\midrule
\multicolumn{10}{l}{\textbf{Skywork-Reward-v2-Qwen3-8B}} \\
No-stop (all)
& 0.3345 & 0.3448 & 0.3538 & 0.3636 & 0.3732
& 0.3792 & 0.3871 & 0.3895 & 0.3945 \\
Early stopping (all)
& 0.3345 & 0.3442 & 0.3526 & 0.3613 & 0.3700
& 0.3751 & 0.3834 & 0.3874 & 0.3937 \\
No-stop (hack)
& 0.1571 & 0.1620 & 0.1602 & 0.1624 & 0.1606
& 0.1526 & 0.1416 & 0.1072 & 0.0487 \\
Early stopping (hack)
& 0.1571 & 0.1622 & 0.1608 & 0.1618 & 0.1626
& 0.1560 & 0.1509 & 0.1262 & 0.0855 \\
Cost saved
& 2.8\% & 7.5\% & 12.8\% & 17.6\% & 21.3\%
& 23.9\% & 25.8\% & 27.1\% & 27.9\% \\

\midrule
\multicolumn{10}{l}{\textbf{allenai/Llama-3.1-8B-Instruct-RM-RB2}} \\
No-stop (all)
& 0.3150 & 0.3202 & 0.3253 & 0.3303 & 0.3338
& 0.3370 & 0.3398 & 0.3405 & 0.3432 \\
Early stopping (all)
& 0.3150 & 0.3195 & 0.3247 & 0.3285 & 0.3329
& 0.3366 & 0.3400 & 0.3427 & 0.3501 \\
No-stop (hack)
& 0.1902 & 0.1877 & 0.1859 & 0.1799 & 0.1736
& 0.1636 & 0.1466 & 0.1135 & 0.0568 \\
Early stopping (hack)
& 0.1902 & 0.1867 & 0.1872 & 0.1829 & 0.1814
& 0.1767 & 0.1675 & 0.1482 & 0.1161 \\
Cost saved
& 3.6\% & 9.3\% & 15.6\% & 21.0\% & 25.0\%
& 27.8\% & 29.7\% & 31.1\% & 32.0\% \\

\bottomrule
\end{tabular}
\end{table*}

\subsection{Analysis of Reward Distribution}
The results of Section~\ref{sec:experimental results}, which examined what kind of distribution the reward distribution of the reward model follows, are shown in Table~\ref{tab:experiment:1}, and the statistics of the reward-value data are reported in Table~\ref{tab:experiment:2}.

\begin{table}[htb]
\caption{Properties of the reward distributions for correct and incorrect answers. $c, w$ are subscripts denoting correct and incorrect answers, respectively, and $\mu, \sigma^2, \gamma, T$ denote the mean, variance, skewness, and tail ratio (the ratio of the right tail to the left tail) of the reward distribution, respectively. In all cases, correct answers are assigned higher reward values on average than incorrect answers. On the other hand, except for Sky-Qwen-4B, the reward models have higher variance for the rewards of incorrect answers than for those of correct answers across all three datasets.}\label{tab:experiment:2}
\centering
\scalebox{0.8}{
\begin{tabular}{l|cc|cc|cc|cc}
\toprule
RM & $\mu_c$ & $\mu_w$ & $\sigma^2_c$ & $\sigma^2_w$ & $\gamma_c$ & $\gamma_w$ & $T_c$ & $T_w$ \\
\midrule
\multicolumn{9}{l}{\textit{GSM8K\_Llama-3-8B-Instruct}} \\
RAMO & -2.8 & -4 & 0.49 & 1.4 & -1.1 & -0.97 & 0.77 & 0.75 \\
allenai-RB2 & 0.42 & -0.86 & 0.37 & 1.3 & -1.6 & -1.2 & 0.64 & 0.57 \\
Sky-Llama-8B & 11 & 5.7 & 22 & 28 & 0.32 & -0.41 & 1.4 & 0.89 \\
Sky-Qwen-1.7B & 4.2 & 1.1 & 1.9 & 5.3 & -0.98 & -0.77 & 0.71 & 0.72 \\
Sky-Qwen-4B & 11 & 7.8 & 1.9 & 6.3 & -1 & -1.1 & 0.69 & 0.67 \\
Sky-Qwen-8B & 6.4 & 2.5 & 2.2 & 7.5 & -1.2 & -0.89 & 0.72 & 0.71 \\
\midrule
\multicolumn{9}{l}{\textit{MATH\_Llama-3-8B-Instruct}} \\
RAMO & -2.1 & -3.2 & 2 & 4.3 & -1.6 & -1.3 & 0.75 & 0.65 \\
allenai-RB2 & -0.069 & -1 & 0.87 & 1.8 & -1.1 & -0.89 & 0.69 & 0.68 \\
Sky-Llama-8B & 18 & 11 & 53 & 74 & -0.22 & 0.41 & 1.3 & 1.3 \\
Sky-Qwen-1.7B & 6.1 & 3.4 & 3.7 & 8.1 & -1 & -0.8 & 0.64 & 0.68 \\
Sky-Qwen-4B & 17 & 11 & 35 & 22 & 0.033 & 0.39 & 1.6 & 1.1 \\
Sky-Qwen-8B & 8.3 & 4.7 & 7.1 & 12 & -0.94 & -0.89 & 0.7 & 0.66 \\
\midrule
\multicolumn{9}{l}{\textit{MATH\_Pythia-12B}} \\
RAMO & -5.4 & -7.3 & 3.7 & 6.9 & -0.38 & -0.41 & 0.83 & 0.76 \\
allenai-RB2 & -3.8 & -5.3 & 2.3 & 2.5 & 0.022 & 0.35 & 1.1 & 1.3 \\
Sky-Llama-8B & 0.29 & -6.1 & 42 & 48 & -0.032 & -0.0075 & 0.88 & 0.98 \\
Sky-Qwen-1.7B & -0.93 & -4.2 & 9.6 & 11 & -0.13 & 0.24 & 0.91 & 1.2 \\
Sky-Qwen-4B & 5.3 & 1.2 & 17 & 15 & 0.15 & 0.33 & 0.97 & 1.3 \\
Sky-Qwen-8B & -0.78 & -5.2 & 19 & 20 & -0.2 & 0.3 & 0.88 & 1.3 \\
\bottomrule
\end{tabular}
}
\end{table}

\begin{table}[]
\caption{Goodness of fit of the reward distributions for correct and incorrect answers to representative probability distributions. $c, w$ are subscripts denoting correct and incorrect answers, respectively. Even when the displayed values appear identical, the judgment is based on the internal numerical ranking.}\label{tab:experiment:1}
\centering
\scalebox{0.76}{
\begin{tabular}{l|ccccccccc}
\toprule
 & normal & student\_t & skew\_normal & skew\_t & gen\_normal & logistic & laplace & johnson\_su & gumbel \\
\midrule
\multicolumn{10}{l}{\textit{RAMO}} \\
AD pass rate (c) $\uparrow$ & 0.47 & 0.53 & \underline{0.74} & \textbf{0.81} & 0.54 & 0.44 & 0.23 & \underline{0.67} & 0.090 \\
AD pass rate (w) $\uparrow$ & 0.047 & 0.16 & \underline{0.31} & \textbf{0.35} & 0.15 & 0.062 & 0.047 & \underline{0.34} & 0.0078 \\
AIC mean (c) $\downarrow$ & 740 & 740 & \textbf{730} & \underline{730} & 740 & 740 & 750 & \underline{730} & 790 \\
AIC mean (w) $\downarrow$ & 1400 & 1400 & \textbf{1400} & \underline{1400} & 1400 & 1400 & 1400 & \underline{1400} & 1500 \\
BIC mean (c) $\downarrow$ & 750 & 750 & \textbf{740} & \underline{740} & 750 & 750 & 760 & \underline{740} & 800 \\
BIC mean (w) $\downarrow$ & 1400 & 1400 & \textbf{1400} & \underline{1400} & 1400 & 1400 & 1500 & \underline{1400} & 1500 \\
log-lik mean (c) $\uparrow$ & -2.0 & -2.0 & \underline{-2.0} & \underline{-2.0} & -2.0 & -2.0 & -2.1 & \textbf{-2.0} & -2.2 \\
log-lik mean (w) $\uparrow$ & -2.4 & -2.4 & \underline{-2.3} & \underline{-2.3} & -2.3 & -2.4 & -2.4 & \textbf{-2.3} & -2.5 \\
\midrule
\multicolumn{10}{l}{\textit{allenai-RB2}} \\
AD pass rate (c) $\uparrow$ & 0.37 & 0.36 & \underline{0.54} & \underline{0.50} & \textbf{0.58} & 0.26 & 0.14 & 0.46 & 0.23 \\
AD pass rate (w) $\uparrow$ & 0.21 & 0.27 & \underline{0.65} & \underline{0.64} & 0.30 & 0.17 & 0.031 & \textbf{0.66} & 0.20 \\
AIC mean (c) $\downarrow$ & 660 & 660 & \underline{640} & 650 & \underline{650} & 660 & 680 & \textbf{640} & 680 \\
AIC mean (w) $\downarrow$ & 1100 & 1100 & \textbf{1100} & \underline{1100} & 1100 & 1100 & 1100 & \underline{1100} & 1100 \\
BIC mean (c) $\downarrow$ & 660 & 670 & \underline{650} & 660 & \underline{660} & 670 & 690 & \textbf{650} & 690 \\
BIC mean (w) $\downarrow$ & 1100 & 1100 & \textbf{1100} & 1100 & \underline{1100} & 1100 & 1100 & \underline{1100} & 1100 \\
log-lik mean (c) $\uparrow$ & -1.8 & -1.8 & \underline{-1.8} & -1.8 & \underline{-1.8} & -1.8 & -1.8 & \textbf{-1.7} & -1.8 \\
log-lik mean (w) $\uparrow$ & -1.8 & -1.8 & \textbf{-1.8} & \underline{-1.8} & -1.8 & -1.8 & -1.9 & \underline{-1.8} & -1.8 \\
\midrule
\multicolumn{10}{l}{\textit{Sky-Llama-8B}} \\
AD pass rate (c) $\uparrow$ & 0.40 & 0.50 & \underline{0.59} & \textbf{0.65} & 0.54 & 0.38 & 0.31 & \underline{0.59} & 0.13 \\
AD pass rate (w) $\uparrow$ & 0.45 & 0.48 & \textbf{0.67} & \underline{0.63} & 0.52 & 0.34 & 0.031 & \underline{0.62} & 0.0078 \\
AIC mean (c) $\downarrow$ & 1200 & 1200 & \underline{1200} & \underline{1200} & 1200 & 1200 & 1200 & \textbf{1100} & 1200 \\
AIC mean (w) $\downarrow$ & 2000 & 2000 & \textbf{2000} & 2000 & \underline{2000} & 2000 & 2000 & \underline{2000} & 2000 \\
BIC mean (c) $\downarrow$ & 1200 & 1200 & \underline{1200} & 1200 & \underline{1200} & 1200 & 1200 & \textbf{1100} & 1200 \\
BIC mean (w) $\downarrow$ & \underline{2000} & 2000 & \textbf{2000} & 2000 & \underline{2000} & 2000 & 2000 & 2000 & 2000 \\
log-lik mean (c) $\uparrow$ & -3.2 & -3.2 & \underline{-3.2} & \underline{-3.2} & -3.2 & -3.2 & -3.2 & \textbf{-3.1} & -3.3 \\
log-lik mean (w) $\uparrow$ & -3.3 & -3.3 & \textbf{-3.3} & \underline{-3.3} & -3.3 & -3.3 & -3.4 & \underline{-3.3} & -3.4 \\
\midrule
\multicolumn{10}{l}{\textit{Sky-Qwen-1.7B}} \\
AD pass rate (c) $\uparrow$ & 0.45 & 0.44 & \textbf{0.65} & 0.56 & \underline{0.60} & 0.38 & 0.23 & \underline{0.59} & 0.15 \\
AD pass rate (w) $\uparrow$ & 0.23 & 0.27 & \underline{0.43} & 0.35 & \textbf{0.45} & 0.13 & 0 & \underline{0.40} & 0.078 \\
AIC mean (c) $\downarrow$ & 920 & 920 & \textbf{910} & 920 & \underline{920} & 930 & 940 & \underline{910} & 950 \\
AIC mean (w) $\downarrow$ & 1600 & 1600 & \textbf{1600} & 1600 & \underline{1600} & 1600 & 1600 & \underline{1600} & 1600 \\
BIC mean (c) $\downarrow$ & \underline{930} & 930 & \textbf{920} & 930 & \underline{920} & 930 & 950 & 930 & 960 \\
BIC mean (w) $\downarrow$ & \underline{1600} & 1600 & \textbf{1600} & 1600 & \underline{1600} & 1600 & 1600 & 1600 & 1600 \\
log-lik mean (c) $\uparrow$ & -2.5 & -2.5 & \textbf{-2.5} & -2.5 & \underline{-2.5} & -2.5 & -2.6 & \underline{-2.5} & -2.6 \\
log-lik mean (w) $\uparrow$ & -2.6 & -2.6 & \textbf{-2.6} & -2.6 & \underline{-2.6} & -2.6 & -2.7 & \underline{-2.6} & -2.6 \\
\midrule
\multicolumn{10}{l}{\textit{Sky-Qwen-4B}} \\
AD pass rate (c) $\uparrow$ & 0.40 & 0.54 & \textbf{0.62} & 0.55 & \underline{0.60} & 0.42 & 0.29 & \underline{0.60} & 0.15 \\
AD pass rate (w) $\uparrow$ & 0.20 & 0.21 & \textbf{0.55} & \underline{0.52} & 0.31 & 0.086 & 0.0078 & \underline{0.47} & 0.20 \\
AIC mean (c) $\downarrow$ & 1000 & 1000 & \underline{1000} & \underline{990} & 1000 & 1000 & 1000 & \textbf{990} & 1000 \\
AIC mean (w) $\downarrow$ & 1600 & 1600 & \textbf{1600} & \underline{1600} & 1600 & 1600 & 1700 & \underline{1600} & 1600 \\
BIC mean (c) $\downarrow$ & 1000 & 1000 & \textbf{1000} & 1000 & \underline{1000} & 1000 & 1000 & \underline{1000} & 1000 \\
BIC mean (w) $\downarrow$ & 1600 & 1700 & \textbf{1600} & 1600 & \underline{1600} & 1700 & 1700 & \underline{1600} & 1700 \\
log-lik mean (c) $\uparrow$ & -2.7 & -2.7 & \underline{-2.7} & \underline{-2.7} & -2.7 & -2.7 & -2.8 & \textbf{-2.7} & -2.8 \\
log-lik mean (w) $\uparrow$ & -2.7 & -2.7 & \textbf{-2.7} & \underline{-2.7} & -2.7 & -2.7 & -2.8 & \underline{-2.7} & -2.7 \\
\midrule
\multicolumn{10}{l}{\textit{Sky-Qwen-8B}} \\
AD pass rate (c) $\uparrow$ & 0.41 & 0.42 & \underline{0.56} & 0.50 & \textbf{0.62} & 0.31 & 0.22 & \underline{0.54} & 0.14 \\
AD pass rate (w) $\uparrow$ & 0.086 & 0.094 & \underline{0.41} & \underline{0.34} & 0.28 & 0.023 & 0.016 & \textbf{0.41} & 0.16 \\
AIC mean (c) $\downarrow$ & 1000 & 1000 & \textbf{1000} & 1000 & \underline{1000} & 1000 & 1100 & \underline{1000} & 1100 \\
AIC mean (w) $\downarrow$ & 1700 & 1700 & \textbf{1700} & 1700 & \underline{1700} & 1700 & 1800 & \underline{1700} & 1700 \\
BIC mean (c) $\downarrow$ & 1000 & 1100 & \textbf{1000} & 1000 & \underline{1000} & 1100 & 1100 & \underline{1000} & 1100 \\
BIC mean (w) $\downarrow$ & 1700 & 1700 & \textbf{1700} & 1700 & \underline{1700} & 1800 & 1800 & \underline{1700} & 1700 \\
log-lik mean (c) $\uparrow$ & -2.9 & -2.9 & \textbf{-2.8} & -2.8 & \underline{-2.8} & -2.9 & -2.9 & \underline{-2.8} & -2.9 \\
log-lik mean (w) $\uparrow$ & -2.9 & -2.9 & \textbf{-2.9} & -2.9 & \underline{-2.9} & -2.9 & -3.0 & \underline{-2.9} & -2.9 \\
\bottomrule
\end{tabular}

}
\end{table}

\subsection{Full Results of Scaling Law Experiments}
In addition, Figure~\ref{fig:scaling theory all} shows the full results of the experiment verifying the scaling law using synthetic data, demonstrating that the theory explains the experimental results well regardless of the parameters of the normal distribution.
\begin{figure}[htb]
\includegraphics[width=\linewidth]{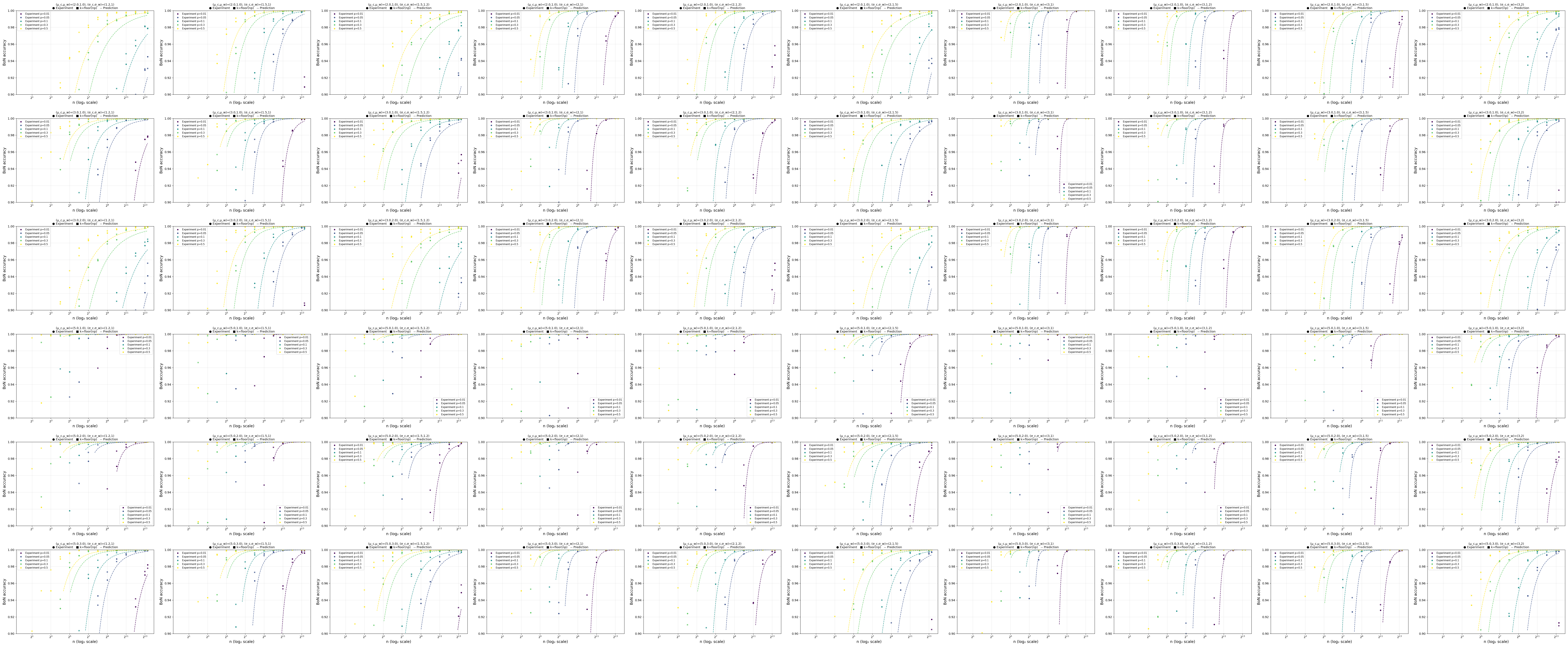}
  \caption{
    Full results of the experiment verifying the scaling law using synthetic data.}
  \label{fig:scaling theory all}
\end{figure}

\subsection{An Experimental Comparison of Early Stopping Methods}
We experimentally compare our Tail-Tension Early Stopping with early stopping methods proposed in \cite{DBLP:journals/corr/abs-2510-01394} but do not compare against \cite{DBLP:journals/corr/abs-2510-15945} and \cite{DBLP:journals/corr/abs-2503-00031}, because the former is not an anytime algorithm, requiring total budget information beforehand and the latter requires training an additional calibration score.

\textbf{Experimental Settings.} The methods \texttt{ucb\_t090}, \texttt{ucb\_t095}, \texttt{ucb\_t099}, \texttt{ucb\_c1e-3}, and \texttt{ucb\_c2e-4} are variants proposed by \cite{DBLP:journals/corr/abs-2510-01394}. The \texttt{ucb\_t090}, \texttt{ucb\_t095}, \texttt{ucb\_t099} correspond to the Target Acceptance Rate Variant described in Section 4.3, with $\tau_{\text{target}} = 0.90, 0.95, 0.99$, independently. The \texttt{ucb\_c1e-3}, and \texttt{ucb\_c2e-4} correspond to the variants of Algorithm 1 in the paper with $c=10^{-3}$ and $2\times 10^{-4}$. The parameter $t_{\mathrm{min}}$ in \cite{DBLP:journals/corr/abs-2510-01394} corresponds to $n_{\mathrm{th}}$ in our method. To ensure a fair comparison, we used the same value as $n_{\mathrm{th}}=t_{\mathrm{min}}=32$ Finally, because the maximum total computational budget was fixed at in these experiments and can not fully consumed , we report entries for which no accuracy value could be obtained as N/A.

Since different early-stopping methods use different amounts of computation, we compare accuracy under the same effective computational budget, allowing computation saved by early stopping on one sample to be reused for other samples. For each method, we linearly interpolate between adjacent points on its accuracy–budget curve and report the interpolate values at budgets of the form $n=2^{x}$.

\textbf{Results.} The result is shown in the table~\ref{tab:early stopping comparison}. The result shows that our method is comparable or outperforms all five variants at the same effective budget, especially for larger effective budget settings. 
\begin{table}[htb]
\caption{Experimental comparison of our Tail-Tension Early Stopping with \cite{DBLP:journals/corr/abs-2510-01394}.}\label{tab:early stopping comparison}
  \centering
  \begin{tabular}{rcccccc}
    \toprule
    Cost & Ours & \texttt{ucb\_t090} & \texttt{ucb\_t095} & \texttt{ucb\_t099}
         & \texttt{ucb\_c1e-3} & \texttt{ucb\_c2e-4} \\
    \midrule
       4 & \textbf{0.288} & 0.285 & 0.285 & 0.285 & 0.285 & 0.285 \\
       8 & \textbf{0.299} & 0.298 & 0.298 & 0.298 & 0.298 & 0.298 \\
      16 & \textbf{0.309} & 0.308 & 0.308 & 0.308 & 0.308 & 0.308 \\
      32 & 0.316 & 0.317 & 0.317 & 0.316 & \textbf{0.318} & \textbf{0.318} \\
      64 & 0.322 & \textbf{0.323} & \textbf{0.323} & \textbf{0.323} & 0.318 & 0.318 \\
     128 & \textbf{0.327} & 0.324 & 0.325 & 0.326 & 0.317 & 0.317 \\
     256 & \textbf{0.332} & 0.323 & 0.324 & 0.325 & 0.317 & 0.316 \\
     512 & \textbf{0.335} & 0.322 & 0.323 & 0.324 & 0.316 & 0.315 \\
    1024 & \textbf{0.339} & 0.321 & 0.321 & 0.322 & N/A & 0.315 \\
    2048 & 0.342 & N/A & N/A & N/A & N/A & N/A \\
    4096 & 0.345 & N/A & N/A & N/A & N/A & N/A \\
    8192 & N/A & N/A & N/A & N/A & N/A & N/A \\
    \bottomrule
  \end{tabular}
\end{table}


\newpage
\clearpage
\section*{NeurIPS Paper Checklist}

\begin{enumerate}

\item {\bf Claims}
    \item[] Question: Do the main claims made in the abstract and introduction accurately reflect the paper's contributions and scope?
    \item[] Answer: \answerYes{} 
    \item[] Justification:  The claims in the abstract and introduction match the theoretical analysis, proposed method, and experiments, and are limited to verifiable problems with reward-model-based selection.

\item {\bf Limitations}
    \item[] Question: Does the paper discuss the limitations of the work performed by the authors?
    \item[] Answer: \answerYes{} 
    \item[] Justification: The paper discusses limitations in Section~\ref{sec:conclusion}.

\item {\bf Theory assumptions and proofs}
    \item[] Question: For each theoretical result, does the paper provide the full set of assumptions and a complete (and correct) proof?
    \item[] Answer: \answerYes{} 
    \item[] Justification: The theoretical assumptions are stated in the main text, and complete proofs of the lemmas and theorems are provided in the appendix.

    \item {\bf Experimental result reproducibility}
    \item[] Question: Does the paper fully disclose all the information needed to reproduce the main experimental results of the paper to the extent that it affects the main claims and/or conclusions of the paper (regardless of whether the code and data are provided or not)?
    \item[] Answer: \answerYes{} 
    \item[] Justification: The paper specifies the datasets, reward models, evaluation protocol, simulation procedure, and compute setup used in the experiments.

\item {\bf Open access to data and code}
    \item[] Question: Does the paper provide open access to the data and code, with sufficient instructions to faithfully reproduce the main experimental results, as described in supplemental material?
    \item[] Answer: \answerNo{} 
    \item[] Justification: The paper uses publicly available datasets and models and describes them with URL in the paper, but the full reproduction code is not included in the current submission.

\item {\bf Experimental setting/details}
    \item[] Question: Does the paper specify all the training and test details (e.g., data splits, hyperparameters, how they were chosen, type of optimizer) necessary to understand the results?
    \item[] Answer: \answerYes{} 
    \item[] Justification: The paper describes the datasets, reward models, sampling protocol, number of trials, and hardware used for the main experiments.

\item {\bf Experiment statistical significance}
    \item[] Question: Does the paper report error bars suitably and correctly defined or other appropriate information about the statistical significance of the experiments?
    \item[] Answer: \answerYes{} 
    \item[] Justification: The paper reports standard deviation in the experiments of the proposed method.

\item {\bf Experiments compute resources}
    \item[] Question: For each experiment, does the paper provide sufficient information on the computer resources (type of compute workers, memory, time of execution) needed to reproduce the experiments?
    \item[] Answer: \answerYes{} 
    \item[] Justification: The paper reports the GPU type and memory used for reward-model inference and describes the computationally lightweight simulation setup.
    
\item {\bf Code of ethics}
    \item[] Question: Does the research conducted in the paper conform, in every respect, with the NeurIPS Code of Ethics \url{https://neurips.cc/public/EthicsGuidelines}?
    \item[] Answer: \answerYes{} 
    \item[] Justification: The work uses public datasets and models, does not involve human-subject experiments or private data, and follows the NeurIPS Code of Ethics.

\item {\bf Broader impacts}
    \item[] Question: Does the paper discuss both potential positive societal impacts and negative societal impacts of the work performed?
    \item[] Answer: \answerNA{} 
    \item[] Justification: Our analysis is a general theory for understanding Best-of-N sampling. The societal impacts depend heavily on the specific use cases and deployment contexts. Therefore, a detailed discussion of societal impacts falls outside the scope of this work.
    
\item {\bf Safeguards}
    \item[] Question: Does the paper describe safeguards that have been put in place for responsible release of data or models that have a high risk for misuse (e.g., pre-trained language models, image generators, or scraped datasets)?
    \item[] Answer: \answerNA{} 
    \item[] Justification: The paper does not release new high-risk models, scraped datasets, or other assets requiring additional safeguards.

\item {\bf Licenses for existing assets}
    \item[] Question: Are the creators or original owners of assets (e.g., code, data, models), used in the paper, properly credited and are the license and terms of use explicitly mentioned and properly respected?
    \item[] Answer: \answerNo{} 
    \item[] Justification: The paper cites the existing datasets and models and following the licenses, but does not yet explicitly list their licenses or terms of use.

\item {\bf New assets}
    \item[] Question: Are new assets introduced in the paper well documented and is the documentation provided alongside the assets?
    \item[] Answer: \answerNA{} 
    \item[] Justification: The paper does not release a new dataset, model checkpoint, benchmark, or other standalone asset.

\item {\bf Crowdsourcing and research with human subjects}
    \item[] Question: For crowdsourcing experiments and research with human subjects, does the paper include the full text of instructions given to participants and screenshots, if applicable, as well as details about compensation (if any)? 
    \item[] Answer: \answerNA{} 
    \item[] Justification:  The paper does not involve crowdsourcing or new research with human subjects.

\item {\bf Institutional review board (IRB) approvals or equivalent for research with human subjects}
    \item[] Question: Does the paper describe potential risks incurred by study participants, whether such risks were disclosed to the subjects, and whether Institutional Review Board (IRB) approvals (or an equivalent approval/review based on the requirements of your country or institution) were obtained?
    \item[] Answer: \answerNA{} 
    \item[] Justification: The paper does not involve human-subject research, so IRB approval is not applicable.

\item {\bf Declaration of LLM usage}
    \item[] Question: Does the paper describe the usage of LLMs if it is an important, original, or non-standard component of the core methods in this research? Note that if the LLM is used only for writing, editing, or formatting purposes and does \emph{not} impact the core methodology, scientific rigor, or originality of the research, declaration is not required.
    \item[] Answer: \answerYes{} 
    \item[] Justification: This work investigates Best-of-n sampling of LLMs and proposes a new method for improving their performance. The methodology and experimental settings are described in Sections 5 and 6.

\end{enumerate}

\end{document}